\documentclass{article}
\usepackage[a4paper, left=3.2cm, right=3.2cm, top=2cm, bottom=2cm]{geometry}
\pdfoutput=1
\input{commands.sty}
\input{packages.sty}

\usepackage[tiny,compact]{titlesec}
\usepackage[standard]{ntheorem}
\usepackage[hidelinks]{hyperref}
\usepackage{cleveref}
\usepackage{tocloft}
	
\newcommand{\keywords}[1]{}
\newcommand{\qed}{\hfill $\Box$}
\renewenvironment{abstract}
{
	\begin{center}
	\bfseries \abstractname\vspace{-.5em}\vspace{0pt}
	\end{center}
}%
{\vspace{3ex}}

\begin{document}
\author{\normalsize Jannik Presberger$^1$ $\cdot$ 
	Rashmiparvathi Keshara$^2$ $\cdot$ 
	David Stein$^1$ $\cdot$
	\\[-0.4ex]\normalsize
	Yung Hae Kim$^2$ $\cdot$ 
	Anne Grapin-Botton$^2$ $\cdot$ 
	Bjoern Andres$^{1,3}$}
\title{\bf \Large Correlation Clustering of Organoid Images}
\date{\normalsize $^1$TU Dresden 
	$\cdot$ $^2$Max Planck Institute of Molecular Cell Biology and Genetics $\cdot$
	\\ 
	$^3$Center for Scalable Data Analytics and AI Dresden/Leipzig}
\maketitle
\begin{abstract}
In biological and medical research, scientists now routinely acquire microscopy images of hundreds of morphologically heterogeneous organoids and are then faced with the task of finding patterns in the image collection, i.e., subsets of organoids that appear similar and potentially represent the same morphological class. We adopt models and algorithms for correlating organoid images, i.e., for quantifying the similarity in appearance and geometry of the organoids they depict, and for clustering organoid images by consolidating conflicting correlations. For correlating organoid images, we adopt and compare two alternatives, a partial quadratic assignment problem and a twin network. For clustering organoid images, we employ the correlation clustering problem. Empirically, we learn the parameters of these models, infer a clustering of organoid images, and quantify the accuracy of the inferred clusters, with respect to a training set and a test set we contribute of state-of-the-art light microscopy images of organoids clustered manually by biologists.
\end{abstract}

\tableofcontents\clearpage
\section{Introduction}

Organoids are biological in vitro model systems used as tools to study organogenesis including morphological development, adult tissue homeostasis, function, disease manifestation, and drug discovery \cite{huch-2017,lukonin-2021,lancaster-2014,lewis-2021,keshara-2022,lee-2022}.
Thanks to recent progress in organoid biology and imaging \cite{rossi-2018,bremond-2021,fillioux-2023}, scientists now routinely acquire light microscopy images of hundreds of organoids at once (cf.~\Cref{figure:image-collection}), which reveals morphological heterogeneity in nature of biological systems. 
For various biological questions, it is essential to identify and classify these organoids based on their morphological characteristics \cite{keshara-2022}.
Hence, scientists are faced with the task of finding patterns in the image collection, i.e., subsets of organoid images that appear similar and potentially represent the same morphological class.

In this article, we adopt mathematical models and algorithms for \emph{correlating} organoid images, i.e., for quantifying the similarity in appearance and geometry of the organoids they depict, and for \emph{clustering} organoid images by consolidating conflicting correlations.
For correlating organoid images, we adopt two alternatives, a partial quadratic assignment problem \cite{haller-2022} with five adjustable parameters, and a twin network \cite{chicco-2021} with $10^7$ adjustable parameters.
For clustering organoid images, we employ the correlation clustering problem \cite{demaine-2006} with costs defined by one or the other correlation model.
Empirically, we learn the parameters of these models, infer a clustering of organoid images, and quantify the accuracy of the inferred clusters, with respect to a training set and a test set of state-of-the-art light microscopy images of organoids clustered manually by biologists.
Both the data sets and the complete source code for reproducing the experiments are contributions of this article and are published as supplementary material.

\begin{figure}[t]
\includegraphics[width=\textwidth]{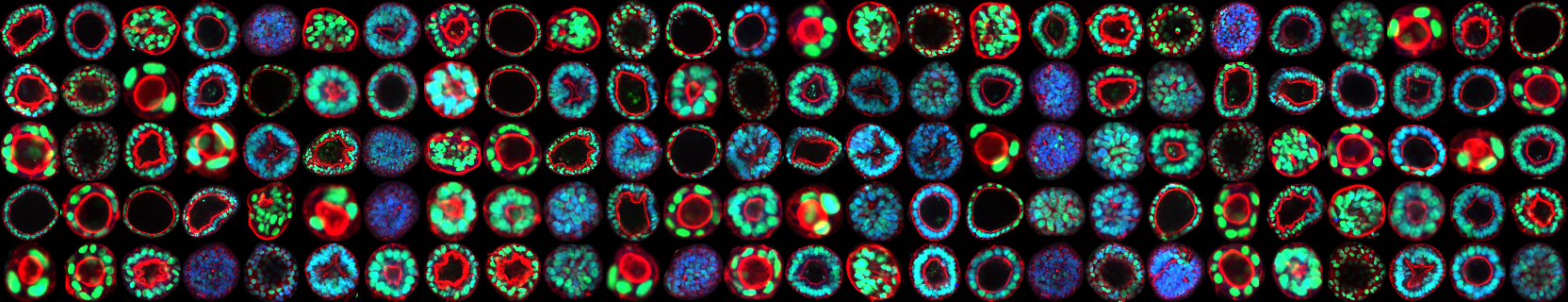}
\caption{Depicted above are 130 images (scaled differently to the same size for illustration) of pancreatic progenitor organoids derived from human pluripotent stem cells. 
These organoids consist of cells expressing a nuclear Green Fluorescent Protein reporter for PDX1 (a pancreatic progenitor marker gene). 
After fixation, the organoids were stained with DAPI (blue) to mark the nucleus and Phalloidin (red) for F-Actin. 
Images were acquired using an automated spinning disc confocal microscope (20x objective of Yokogawa CV7000).}
\label{figure:image-collection}
\end{figure}

\section{Related work}
The task of correlating images of composite objects by bringing the objects in correspondence has been studied comprehensively \cite{cho-2013,min-2019,bourdev-2009}.
The partial quadratic assignment problem, also known as the graph matching problem \cite{torresani-2008}, a relaxation of the quadratic assignment problem \cite{lawler-1963,cela-2013}, is a mathematical abstraction of this task \cite{torresani-2008,hutschenreiter-2021}.
Algorithms for this \textsc{np}-hard problem are an active area of computer vision research \cite{abbas-2024,dlask-2023,abbas-2022,haller-2022,hutschenreiter-2021,swoboda-2019,swoboda-2017};
see \cite{haller-2022} for a recent survey.
In applications to natural images such as \cite{cho-2013,min-2019,bourdev-2009}, one arrives at the partial quadratic assignment problem by fixing a set of key points (of one or several object models) and estimating from images annotated with key points \cite{cho-2013,min-2019,bourdev-2009} the appearance of key points as well as their relative location, e.g., as described in \cite{rolinek-2020,tourani-2023}.
In the application to organoid images, which exhibit fairly heterogeneous objects in their morphologies \cite{grapin-botton-2022}, we cannot fix the set of key points, and we do not know the object models, as we do not know the morphological classes.
Instead, every image is potentially a separate organoid class.
Our approach to organoid images differs from published approaches to natural images in that we fix a set of key point \emph{classes} and an algorithm that extracts from each organoid image key points and a separate organoid model that we then match with the organoid models extracted from all other organoid images.
We estimate the cost coefficients of the partial quadratic assignment problems not from key point annotations but from a clustering of an image collection.
Compared to \cite{rolinek-2020,tourani-2023}, our object models, learning algorithm and inference algorithm are less sophisticated, but the set of object models is larger and data-dependent and gives rise to partial assignment problems with complete bipartite graphs that are hard to solve exactly.

The task of clustering images based on visual similarity is typically addressed by first estimating an embedding of single images into a metric space \cite{bhatia-2015,wang-2017}, and then clustering these points by some form of metric-based clustering, e.g., $k$-means clustering.
Another approach to clustering images is to first estimate a map from any \emph{pair} of images to a real number that is positive for images in the same cluster and negative for images in distinct clusters (at loss zero), and to then cluster a set of images by correlation clustering \cite{bansal-2004,demaine-2006}.
Examples of this approach include \cite{ho-2021,rouhi-2021,marra-2016}.
A combination of both approaches is introduced in \cite{abbas-2023}.
One reason why forensics researchers \cite{rouhi-2021,marra-2016} and also we here adopt the correlation clustering problem is its unbiasedness with respect to clusterings a priori \cite{kappes-2016,alush-2016,andres-2011}.
Thanks to recent progress \cite{abbas-2023,jovanovic-2023,abbas-2022-rama,veldt-2022,sorensen-2020} toward algorithms for the \textsc{np}-hard correlation clustering problem and the closely related clique partition problem and multicut problem, some instances not accessible previously can now be solved exactly; see \cite{sorensen-2023,swoboda-2023-structured} for recent benchmarks.

The task of extracting images of individual organoids from microscopy images of numerous organoids within a field of view is referred to as organoid segmentation \cite{park-2023}.
The images of individual organoids we contribute as supplementary material are extracted from microscopy images using \cite{schmidt-2018} as we describe in \Cref{appendix:organoid-segmentation}.
The task of organoid tracking \cite{matthews-2022} is not addressed in this article.

\section{Models}
\label{section:models}

\begin{figure}[t]
\input{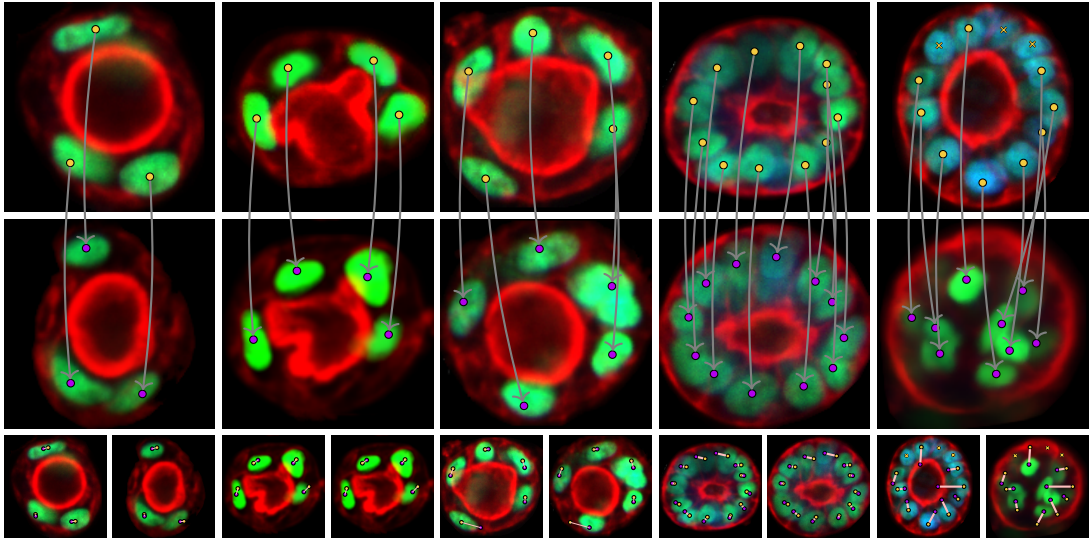}
\vspace{-4ex}
\caption{Depicted above, from left to right, are assignments (gray lines) between key points of five pairs of organoid images. 
Depicted in the bottom row are projections onto the first and the second image.
Depicted in Columns 1--4 are assignments between morphologically similar organoids.
Depicted in the last column is an assignment between dissimilar organoids. 
Note that distances between assigned key points are larger here.
For illustration, images are rotated, and only key points of cell nuclei are shown.}
\label{figure:assignment-example5}
\end{figure}

\subsection{Correlating organoid images by partial quadratic assignment}
\label{section:models-assignment}

The first idea we pursue in order to compare one image of a first organoid to another image of a second organoid is to assign points in the first image to points in the second image such that the images appear similar at assigned points, and such that the assignments of any two pairs of points are geometrically consistent.
This idea is made rigorous below and is illustrated in \Cref{figure:assignment-example5}.

\paragraph{Feasible solutions}
We consider a finite collection $J$ of organoid images.
From any image $j \in J$, we extract a finite set $V_j$ of key points as described in \Cref{appendix:assignment-key-points}.
For any pair $j,k \in J$ of distinct images, we call a relation between $V_j$ and $V_k$ a \emph{feasible assignment} if and only if any point in $V_j$ is related to at most one point in $V_k$, and any point in $V_k$ is related to at most one point in $V_j$.
We encode any feasible assignment by the vector $x \in \{0,1\}^{V_j \times V_k}$ such that for any point $v \in V_j$ and any point $w \in V_k$, we have $x_{vw} = 1$ if and only if $v$ is assigned to $w$.
The set $X_{V_j V_k}$ of the encoding vectors of all feasible assignments is written below, with the above conditions expressed equivalently in the form of linear inequalities.
\begin{align*}
X_{V_j V_k} = \bigg\{ x \in \{0,1\}^{V_j \times V_k} \,\bigg|\,
& \forall v \in V_j \colon \!\! \sum_{w \in V_k} \! x_{vw} \leq 1 
\ \ \text{and}\ \ 
\forall w \in V_k \colon\!\! \sum_{v \in V_j} \! x_{vw} \leq 1 \bigg\}
\end{align*}

\paragraph{Partial quadratic assignment problem}
For any images $j, k \in J$ and any points $v \in V_j$ and $w \in V_k$, we define below a cost $c^{jk}_{vw} \in \mathbb{R}$ that is negative if the point $v$ in the image $j$ appears similar to the point $w$ in the image $k$, and positive if the images appear dissimilar at these points.
For any distinct points $v, v' \in V_j$ and any distinct points $w, w' \in V_k$, we also define below a cost $c^{jk}_{vwv'w'} \in \mathbb{R}$ that is negative if the assignment of the point $v$ in the image $j$ to the point $w$ in the image $k$ is geometrically consistent with the assignment of the point $v'$ in the image $j$ to the point $w'$ in the image $k$, and positive otherwise.
For any feasible solution $x \in X_{V_j V_k}$, we define its cost $\varphi_{c^{jk}}(x) \in \mathbb{R}$ as
\begin{align}
\varphi_{c^{jk}}(x) =\ &
\frac{1 - \lambda}{n^{jk}_1} \sum_{v \in V_j} \sum_{w \in V_k} c^{jk}_{vw} \, x_{vw}
\nonumber\\
& + \frac{\lambda}{n^{jk}_2} \sum_{v \in V_j} \sum_{v' \in V_j \setminus \{v\}} \sum_{w \in V_k} \sum_{w' \in V_k \setminus \{w\}} c^{jk}_{vwv'w'} \, x_{vw} \, x_{v'w'}
\label{eq:assignment-objective}
\end{align}
with the constants $n^{jk}_1 := \min \{|V_j|, |V_k|\}$ and $n^{jk}_2 := \tbinom{n^{jk}_1}{2}$ and a model parameter $\lambda \in (0,1)$.
Note: The cost $c^{jk}_{vwv'w'}$ is payed if and only if $x_{vw} = x_{v'w'} = 1$, i.e., if the point $v$ is assigned to the point $w$, \emph{and} the point $v'$ is assigned to the point $w'$.
Now, the task of finding a feasible assignment of minimum cost takes the form of the partial quadratic assignment problem
\begin{align}
\min \left\{ \varphi_{c^{jk}}(x) \;\middle|\; x \in X_{V_j V_k} \right\}
\enspace .
\label{eq:assignment}
\end{align}

\paragraph{Cost coefficients}
For any organoid image $j \in J$ and any point $v \in V_j$ in the image, we let $a_v^j \in \mathbb{R}^3$ denote the vector of the three color channels of the image $j$ at the point $v$.
We measure the similarity in appearance of a point $v \in V_j$ in an image $j$ and a point $w \in V_k$ in an image $k$ by the metric $d_{vw} := |a_v^j - a_w^k|$.

For any image $j \in J$, we estimate the barycenter $r^j_0 \in \mathbb{R}^2$ of the organoid in the image plane, as well as the extent $\sigma_0^j \in \mathbb{R}^+$ of the organoid, i.e., the furthest $l_2$-distance of any point of the organoid from its barycenter, as described in \Cref{appendix:assignment-barycenter-extent}.
For any point $v \in V_j$, we consider its coordinates $r^j_v \in \mathbb{R}^2$ in the image plane, as well as the distance between $r^j_v$ and the barycenter, relative to the extent of the organoid, i.e., $\sigma_v^j := |r_v^j - r_0^j|/\sigma_0^j$.
For any pair of points $v,v' \in V_j$, we consider the angle $\alpha^j_{vv'} := \angle(r_v^j - r_0^j, r_{v'}^j - r_0^j)$.
For any pair of images $j,k \in J$, any point $v \in V_j$ and any point $w \in V_k$, we define $d'_{vw} := |\sigma_v^j - \sigma_w^k|$.
For any pair of images $j,k \in J$, any distinct points $v,v' \in V_j$ and any distinct points $w,w' \in V_k$, we define $d''_{vwv'w'} := |\alpha^j_{vv'} - \alpha^k_{ww'}|$.
We justify these definitions in \Cref{section:geometric-consistency}.

With model parameters $\delta, \delta', \delta''  \in \mathbb{R}^+$ and $\theta \in (0,1)$, we define the coefficients in \eqref{eq:assignment-objective} as
\begin{align}
c^{jk}_{vw} & := \theta \, (d_{vw} - \delta) + (1 - \theta) \, (d'_{vw} - \delta') \\
c^{jk}_{vwv'w'} & := d''_{vwv'w'} - \delta'' 
\enspace .
\end{align}
According to these definitions, distances $d_{vw}, d'_{vw}, d''_{vwv'w'}$ greater than the respective thresholds $\delta, \delta', \delta''$ result in positive costs associated with the assignments of $v$ to $w$ and $v'$ to $w'$.
Distances less than the thresholds result in negative costs (i.e., rewards) associated with these assignments.

\paragraph{Normalization}
For any pair of images $j, k \in J$ and the costs $\varphi_{c^{jk}}(x^{jk})$ and $\varphi_{c^{kj}}(x^{kj})$ of solutions $x^{jk}$ and $x^{kj}$ to the instances of the partial quadratic assignment problems between these images, we consider the real number
%
\begin{align}
	\phi_{\{j,k\}} := 
	- \frac
	{\min \left\{\varphi_{c^{jk}}(x^{jk}), \varphi_{c^{kj}}(x^{kj}) \right\}}
	{(1 - \lambda) (\theta \delta + (1-\theta) \delta') + \lambda \delta''}
	\enspace .
	\label{eq:correlation-clustering-coefficient}
\end{align}
In \eqref{eq:correlation-clustering-coefficient}, the minimum merely serves the purpose of defining $\phi_{\{j,k\}}$ by a form that is invariant under transposition of $j$ and $k$. 
In theory, $\varphi_{c^{jk}}(x^{jk}) = \varphi_{c^{kj}}(x^{kj})$ because the partial quadratic assignment problem and the cost function we define for this problem are invariant under transposition of $j$ and $k$.
Also in \eqref{eq:correlation-clustering-coefficient}, the denominator is a normalizing constant that only depends on model parameters.
As we show in \Cref{section:geometric-consistency}, the number $\phi_{\{j,k\}}$ is confined to the interval $[0, 1]$ and is \emph{comparable across instances} of the partial quadratic assignment problem.

\subsection{Correlating organoid images by a twin network}
\label{section:models-twin-network}

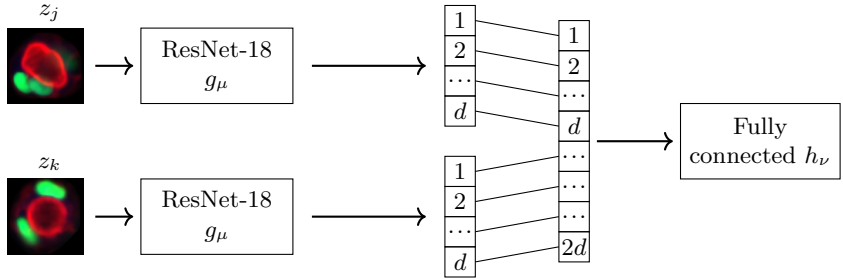
\begin{figure}[t]
	\centering
	\small
	\begin{tikzpicture}	
		\newcommand\w{2.0};
		\newcommand\h{1};
		\newcommand\compacthdots{\hbox to 1em{.\hss.\hss.}}
		\node[align=center] at (-1.2, 1.0) {
			\includegraphics[width=1cm, height=1cm]{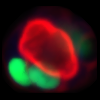}
		};
		\node at (-1.2, 1.7) {$z_j$};
		\draw[thick,->] (-0.6, 1.0) -- ++(0.5, 0.0);
		\node[align=center] at (-1.2, -1.0) {
			\includegraphics[width=1cm, height=1cm]{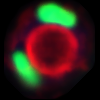}
		};
		\node at (-1.2, -0.3) {$z_k$};
		\draw[thick,->] (-0.6, -1.0) -- ++(0.5, 0.0);
		\draw (0, 0.5) rectangle ++(\w, \h) node[pos=0.5, align=center] {ResNet-18\\$g_\mu$};
		\draw (0, -0.5) rectangle ++(\w, -\h) node[pos=0.5, align=center] {ResNet-18\\$g_\mu$};
		\draw[thick,->] (\w + 0.25, 0.5 + \h/2) -- ++(1.5, 0.0);
		\draw[thick,->] (\w + 0.25, -0.5 - \h/2) -- ++(1.5, 0.0);
		\draw (\w + 2.0, 0.9 + \h / 2) rectangle ++(0.4, 0.4) node[midway] {1};
		\draw (\w + 2.0, 0.5 + \h / 2) rectangle ++(0.4, 0.4) node[midway] {2};
		\draw (\w + 2.0, 0.1 + \h / 2) rectangle ++(0.4, 0.4) node[midway] {\compacthdots};
		\draw (\w + 2.0, -0.3 + \h / 2) rectangle ++(0.4, 0.4) node[midway] {$d$};
		\draw (\w + 2.0, -1.3 - \h / 2) rectangle ++(0.4, 0.4) node[midway] {$d$};
		\draw (\w + 2.0, -0.9 - \h / 2) rectangle ++(0.4, 0.4) node[midway] {\compacthdots};
		\draw (\w + 2.0, -0.5 - \h / 2) rectangle ++(0.4, 0.4) node[midway] {2};
		\draw (\w + 2.0, -0.1 - \h / 2) rectangle ++(0.4, 0.4) node[midway] {1};
		\draw (\w + 3.5, 0.7 + \h / 2) rectangle ++(0.4, 0.4) node[midway] {1};
		\draw (\w + 3.5, 0.3 + \h / 2) rectangle ++(0.4, 0.4) node[midway] {2};
		\draw (\w + 3.5, -0.1 + \h / 2) rectangle ++(0.4, 0.4) node[midway] {\compacthdots};
		\draw (\w + 3.5, -0.5 + \h / 2) rectangle ++(0.4, 0.4) node[midway] {$d$};
		\draw (\w + 2.4, 1.6) -- (\w + 3.5, 1.4);
		\draw (\w + 2.4, 1.2) -- (\w + 3.5, 1.0);
		\draw (\w + 2.4, 0.8) -- (\w + 3.5, 0.6);
		\draw (\w + 2.4, 0.4) -- (\w + 3.5, 0.2);
		\draw (\w + 3.5, -1.1 - \h / 2) rectangle ++(0.4, 0.4) node[midway] {$2d$};
		\draw (\w + 3.5, -0.7 - \h / 2) rectangle ++(0.4, 0.4) node[midway] {\compacthdots};
		\draw (\w + 3.5, -0.3 - \h / 2) rectangle ++(0.4, 0.4) node[midway] {\compacthdots};
		\draw (\w + 3.5, 0.1 - \h / 2) rectangle ++(0.4, 0.4) node[midway] {\compacthdots};
		\draw (\w + 2.4, -1.6) -- (\w + 3.5, -1.4);
		\draw (\w + 2.4, -1.2) -- (\w + 3.5, -1.0);
		\draw (\w + 2.4, -0.8) -- (\w + 3.5, -0.6);
		\draw (\w + 2.4, -0.4) -- (\w + 3.5, -0.2);
		\draw[thick,->] (\w + 4.0, 0.0) -- ++(1.0, 0);
		\draw (\w + 5.1, -\h/2) rectangle ++(\w+0.1, \h) node[pos=0.5, align=center] {Fully\\connected $h_\nu$};
	\end{tikzpicture}
	\caption{In order to map a pair of scaled organoid images $z_j, z_k \in \mathbb{R}^{3 n^2}$ to a real number that is supposed to be positive for images in the same cluster and negative for images in distinct clusters, we learn a twin network
	consisting of a head $g_\mu \colon \mathbb{R}^{3n^2} \to \mathbb{R}^d$ in the form of a ResNet-18 \cite{He2016} with $d = 128$ and adjustable parameters $\mu \in \mathbb{R}^{11,242,176}$, and a base $h_\nu \colon \mathbb{R}^d \times \mathbb{R}^d \to \mathbb{R}$ in the form of one fully connected hidden layer and one output node, with adjustable parameters $\nu \in \mathbb{R}^{33,025}$.}
	\label{figure:siamese-network-sketch}
\end{figure}

The second idea we pursue in order to compare one image of a first organoid to another image of a second organoid is to learn from a collection of organoid images clustered by biologists a map from any \emph{pair} of organoid images to a real number that is positive for images in the same cluster and negative for images in distinct clusters (at loss zero).

To this end, we scale every organoid image to fit into a grid of $n \cdot n$ pixels, with $n = 256$, centered and padded with zeroes in case the image is not square.
For any organoid image $j \in J$, we let $z_j \in \mathbb{R}^{3 n^2}$ denote the vector of the intensities of the three color channels of all pixels in the scaled image. 
For any distinct organoid images $j, k \in J$, we let $z_{jk}$ denote the concatenation of $z_j$ and $z_k$.
We define a function that maps any $z_{jk}$ to a real number that is supposed to be positive for images in the same cluster and negative for images in distinct clusters.
This function has the form of the twin network \cite{chicco-2021} depicted in \Cref{figure:siamese-network-sketch}.
It consists of a head $g_\mu \colon \mathbb{R}^{3n^2} \to \mathbb{R}^d$ in the form of a ResNet-18 \cite{He2016} with $d = 128$ and adjustable parameters $\mu \in \mathbb{R}^{11,242,176}$, and a base $h_\nu \colon \mathbb{R}^d \times \mathbb{R}^d \to \mathbb{R}$ in the form of one fully connected hidden layer and one output node, with adjustable parameters $\nu \in \mathbb{R}^{33,025}$.
Overall, this twin network defines a function $\tau_\theta \colon \mathbb{R}^{6n^2} \to \mathbb{R}$ with adjustable parameters $\theta := (\mu, \nu)$ such that for any pair of organoid images $j,k \in J$, we have $\tau_\theta(z_{j k}) = h_\nu(g_\mu(z_j), g_\mu(z_{k}))$.

\subsection{Correlating organoid images by correlating their histograms}
\label{section:models-histogram}

A third idea we pursue in order to compare one image of a first organoid to another image of a second organoid is to compare the histograms of these images.

To this end, we define a twin network analogous to the one described in \Cref{section:models-twin-network} and depicted in \Cref{figure:siamese-network-sketch}, except that here, the input consists of the histograms $h_j, h_k \in \{0,\dots,255\}^3$ of two images $j,k \in J$.
The head of this twin network has the form of a ResNet-18 \cite{He2016} with $d = 128$, one-dimensional convolutions and adjustable parameters $\mu' \in \mathbb{R}^{3,910,464}$.
The base has exactly the form described in \Cref{section:models-twin-network}, with adjustable parameters $\nu' \in \mathbb{R}^{33,025}$.
We write $\theta' = (\mu', \nu')$ for the parameters of this network and write $\tau'_{\theta'}$ for the function defined by this network.

A fourth idea we implement as a baseline in order to compare two images of organoids by comparing their histograms is to compute the Hellinger distance \cite{hellinger-1909} between these histograms.

\subsection{Clustering organoid images by correlation clustering}
\label{section:models-clustering}
The idea we pursue in order to cluster a collection $J$ of organoid images is to consider for any distinct images $j,k \in J$ a cost $q_{\{j,k\}} \in \mathbb{R}$ that is positive for images of similar organoids and negative for images of dissimilar organoids, and to search for a partition (clustering) of the set $J$ that minimizes the sum of costs of those pairs of images that are in distinct clusters.

\paragraph{Feasible solutions}
We consider as feasible solutions all partitions of the set $J$ of organoid images.
We encode any partition $\Pi$ of $J$ by the vector $y \colon \tbinom{J}{2} \to \{0,1\}$ such that for any pair $\{j,k\} \in \tbinom{J}{2}$ of distinct images, we have $y_{\{j,k\}} = 1$ if and only if $j$ and $k$ are in distinct clusters, i.e., if $\forall U \in \Pi \colon \{j,k\} \not\subseteq U$.
The set $Y_J$ of the encoding vectors of all partitions is written below, with transitivity expressed in the form of linear inequalities.
\begin{align*}
Y_J = \bigg\{
y \colon\! \tbinom{J}{2} \!\to\! \{0,1\}
\,\bigg|\,
\forall j \in J \,
\forall k \in J \!\setminus\! \{j\} \,
\forall l \in J \!\setminus\! \{j,k\}
\colon
y_{\{j, l\}} \leq y_{\{j,k\}} + y_{\{k,l\}}
\bigg\}
\end{align*}

\paragraph{Correlation clustering problem}
For any pair of distinct images $j,k \in J$, we define below a cost $q_{\{j,k\}} \in \mathbb{R}$ associated with $j$ and $k$ being in distinct clusters.
For any feasible solution $y \in Y_J$, we define its cost $\xi_q(y) \in \mathbb{R}$ as
\begin{align}
\xi_q(y) = \sum_{\{j,k\} \in \tbinom{J}{2}} q_{\{j,k\}} \, y_{\{j,k\}}
\enspace .
\end{align}

Hence, the task of finding an optimal partition of the image collection $J$ assumes the form of the correlation clustering problem
\begin{align}
	\label{eq:clustering-objective}
\min \left\{
\xi_{q}(y)
\ \middle|\ 
y \in Y_J 
\right\}
\enspace .
\end{align}
Note: The number, size and size ratio of clusters is neither constrained nor assigned a cost explicitly in this problem.
Instead, these properties are determined by its solutions.

\paragraph{Cost coefficients}
We consider four alternatives of defining the cost $q_{\{j,k\}}$ associated with a pair of distinct images $j,k \in J$ being in distinct clusters, one for each method we consider for correlating organoid images:

When correlating organoid images by partial quadratic assignment, as described in \Cref{section:models-assignment}, we define $q_{\{j, k\}} := \phi_{\{j, k\}} - \delta'''$, with an additional model parameter $\delta'''\in \left[0, 1\right]$.
When correlating organoid images by a twin network $\tau_\theta$, as described in \Cref{section:models-twin-network}, we define $q_{\{j, k\}} := \frac{1}{2} (\tau_\theta(z_{j k}) + \tau_\theta(z_{k j}))$.
When correlating organoid images $j$ and $k$ via their histograms $h_j$ and $h_k$ by a twin network $\tau'_{\theta'}$, as described in \Cref{section:models-histogram}, we define $q_{\{j, k\}} := \frac{1}{2} (\tau'_{\theta'}(h_{j k}) + \tau'_{\theta'}(h_{k j}))$ with $h_{jk}$ the concatenation of the histograms $h_j$ and $h_k$.
When correlating organoid images $j$ and $k$ via their histograms $h_j$ and $h_k$ by the Hellinger distance $d_H(h_j, h_k)$, as described in \Cref{section:models-histogram}, we define $q_{\{j, k\}} := 1 - d_H(h_j, h_k) - \delta'''$, with an additional model parameter $\delta''' \in [0,1]$.

\section{Algorithms}

\subsection{Inference}
\label{section:algorithms-inference}

\paragraph{Partial quadratic assignment problem}
We search for feasible solutions to the instances defined in \Cref{section:models-assignment} of the partial quadratic assignment problem by means of a local search algorithm that exploits the geometry of these instances:

For any two images $j,k \in J$, we assume (hypothetically) that a geometrically consistent feasible assignment $x \in X_{V_j V_k}$ as defined in \Cref{section:geometric-consistency} exists and consider the transformation of the image plane defined by $x$ according to \Cref{definition:consistency}.
This transformation is composed of a translation, a rotation and a scaling.
We estimate the translation from the estimates of the barycenters of the organoids.
We estimate the scaling from the estimates of the extents of the organoids.
For any rotation angle $\gamma$, we thus obtain one transformation $f_\gamma$.

For any angle $\gamma \in \{\frac{2 \pi n}{N}\}_{n \in \{0,\ldots,N-1\}}$, we construct a feasible assignment $x^{\gamma} \in X_{V_{j}V_{k}}$ as follows:
For any $v \in V_j$, we constrain the set of candidates to a subset $V_{jk} \subseteq V_k$ of cardinality $|V_{jk}| = \lfloor |V_k|/M \rfloor$ for which $\sum_{w \in V_{jk}} \| f_\gamma(r^j_v) - r^k_w\|_{2}$ is minimal.
Starting from $x = 0$, we greedily assign pairs of unassigned points $v \in V_j$ and $w \in V_{jk}$ for which switching $x^\gamma_{vw}$ to one reduces $\varphi_{c^{jk}}(x^\gamma)$ maximally, until even the best additional assignment would increase $\varphi_{c^{jk}}(x^\gamma)$.
We output $x^\gamma$ for a $\gamma$ that minimizes $\varphi_{c^{jk}}(x^\gamma)$.
In our experiments, $N = 75$ and $M = 10$.

\paragraph{Correlation clustering problem}
We solve globally all instances of the correlation clustering problem for our quantitative experiments, by separating precisely the inequalities from the definition of $Y_J$ in the integer linear programming procedure of Gurobi~\cite{gurobi}.
We solve locally the instances of the correlation clustering problem for our qualitative experiments in \Cref{appendix:experiments-qualitative}, by means of \cite[Algorithm~2]{keuper-2015a}.

\subsection{Learning}
\label{section:algorithms-learning}
\paragraph{Partial quadratic assignment problem}
In order to learn the parameters $\delta, \delta', \delta'', \theta, \lambda$ of the instances defined in \Cref{section:models-assignment} of the partial quadratic assignment problem, as well as the parameter $\delta'''$, we classify pairs of organoid images $j$ and $k$ independently as being of the same organoid class, if $\phi_{\{j,k\}} \geq \delta'''$, and of distinct organoid classes, if $\phi_{\{j,k\}} < \delta'''$.
Our learning objective is to minimize the $F_1$ score of this classification with respect to all pairs of distinct images of the training set, and the decisions by biologists whether the images belong to the same cluster or distinct clusters.
We perform this minimization locally, by simulated annealing:
In Iteration $t = 0$, we initialize the parameters, solve all $\tbinom{|J|}{2}$ assignment problems locally, as described in \Cref{section:algorithms-inference}, classify pairs of images independently, as described above, and compute the $F_1$ score $F_1^{(0)}$.
In every successive iteration $t$, we draw for every parameter a number $u$ from a normal distribution with zero mean and standard deviation $\kappa \in \mathbb{R}$, add $u$ to that parameter, solve all $\tbinom{|J|}{2}$ assignment problems locally, as described in \Cref{section:algorithms-inference}, classify pairs of images independently and compute the $F_1$ score $F_1^{(t)}$.
If $F_1^{(t)} > F_1^{(t-1)}$, we accept the parameter update.
Otherwise, we accept the parameter update with probability $e^{(F_1^{(t)} - F_1^{(t-1)}) / T_t}$, where $T_t \in \mathbb{R}^+$ is called the temperature at iteration $t$ that we adjust according to a geometric cooling schedule $T_t = \beta \cdot T_{t-1}$.
If the parameter update is revoked we maintain $F_1^{(t)} := F_1^{(t-1)}$.
After an iteration limit $t_{max}$ is reached, we output parameters for which the $F_1$ score is maximal over all iterations.
In our experiments, we start with $\delta = \delta' = \delta'' = 0.2$, $\lambda = \theta = 0.5$, $T_0 = 0.3$ and fix $\kappa = 0.1$, $\beta = 0.99, t_{max} = 140 \left(\text{8 hours}\right)$.

\paragraph{Twin network}
In order to learn the parameters $\theta$ of the twin network $\tau_\theta$, our objective is to minimize the logistic loss with respect to all pairs of distinct images of the training set, and the decisions by biologists whether the images belong to the same cluster or distinct clusters.
In our experiments, we solve this problem locally, by means of stochastic gradient descent, with an adaptive learning rate.
More specifically, we employ AdamW~\cite{loshchilov2018decoupled} with mini-batches, each  consisting of $64$ pairs of images, $32$ pairs where both images are from the same cluster, and $32$ pairs where the two images are from distinct clusters.
The initial learning rate is $10^{-4}$.
The number of iterations is $6\,000$. 
Optionally, for data augmentation, we apply the following transformations to images, each with a probability of $\frac{1}{5}$, independently:
Reflection about the horizontal axis, reflection about the vertical axis, rotation about the origin by an angle drawn uniformly at random from $\left[0, 2\pi\right)$.
In order to learn the parameters $\theta'$ of the twin network $\tau'_{\theta'}$, the procedure is exactly analogous, except that we do not augment histograms.

\paragraph{Hellinger Distance}
In order to estimate the threshold $\delta'''\in \left[0, 1\right]$ on the Hellinger Distance, we classify pairs of organoid images $j$ and $k$ independently as being of the same organoid class, if $d_H(h_j, h_k) \leq 1 - \delta'''$, and of distinct organoid classes, if $d_H(h_j, h_k) > 1 - \delta'''$.
Our learning objective is to maximize the $F_1$ score of this classification with respect to a training set of organoid images clustered manually by biologists.
We perform this maximization by exhaustive search over all $\delta''' \in \{\frac{n}{N'}\}_{n \in \{0,\dots,N'\}}$ with $N' = 100$.

\section{Experiments}
\label{section:experiments}

\subsection{Data set of organoid images with annotations}
We report below on experiments with respect to four sets of organoid images that we refer to as \emph{Train-100}, \emph{Test-100}, \emph{Test-30} and \emph{Unlabeled-1000}
(see \Cref{figure:image-collection,figure:clustering_results} and \Cref{appendix:experiments-qualitative}).
Each of these sets consists of light microscopy images of individual organoids we have extracted from larger images by means of \cite{schmidt-2018}, as described in \Cref{appendix:organoid-segmentation}.
\emph{Train-100} and \emph{Test-100} each contain 10 organoid images of each of 10 distinct organoid classes identified by biologists based on the appearance of organoids in the image collection.
\emph{Test-30} contains another 10 images of each of 3 additional organoid classes not contained in \emph{Train-100} or \emph{Test-100}.
We use \emph{Train-100} exclusively for learning.
We use \emph{Test-100} for measuring the accuracy of the classification of image pairs and of the clustering of images, with respect to a set of images not seen during learning, of organoid classes seen during learning.
We use \emph{Test-30} for measuring the accuracy with respect to a set of images not seen during learning, of organoid classes not seen during learning.
We use \emph{Unlabeled-1000} for reporting qualitatively in \Cref{appendix:experiments-qualitative} the clustering on an unlabeled set of 1000 images of organoids of unknown classes.
All four sets of organoid images are included as supplementary material, along with the annotations of the first three.

\subsection{Correlating organoid images}
\label{section:experiments-assignments}

\begin{table}[!b]
	\centering\scriptsize
	\begin{tabular}{ll@{\ \ }rrrrr@{\ \ }rrrrrrrrrr}
	\toprule
	Test & Model & \multicolumn{5}{l}{Classification of pairs}  & \multicolumn{10}{l}{Correlation clustering} \\
	data &       & ACC & PC & RC & PJ & RJ     & RI & VI & VI$_{\textnormal{C}}$ & VI$_{\textnormal{J}}$ & PC & RC & PJ & RJ & F\textsubscript{1}C & F\textsubscript{1}J\\
	\midrule
	100 & PQAP & 95.0 & 97.7 & 96.8 & 70.7 & 77.1 & 97.2 & 0.68 & 0.53 & 0.15 & 97.7 & 99.3 & 92.0 & 76.7 & 98.5 & 83.7 \\
	& TNIa & \textbf{98.9} & \textbf{99.6} & \textbf{99.2} & \textbf{92.5} & \textbf{95.8} & \textbf{99.4} & \textbf{0.14} & \textbf{0.09} & \textbf{0.05} & \textbf{99.6} & \textbf{99.8} & \textbf{97.7} & \textbf{96.0} & \textbf{99.7} & \textbf{96.9} \\
	& TNI & 92.6 & 99.0 & 98.4 & 84.7 & 89.8 & 98.2 & 0.38 & 0.18 & 0.20 & 99.1 & 98.9 & 89.5 & 91.3 & 99.0 & 90.4 \\
	& TNH & 92.6 & 96.6 & 95.2 & 58.0 & 66.7 & 93.0 & 1.58 & 0.92 & 0.67 & 96.0 & 96.3 & 62.0 & 59.8 & 96.2 & 60.9 \\
	& $d_H$ & 90.1 & 95.1 & 93.9 & 46.0 & 51.8 & 91.6 & 1.76 & 1.19 & 0.57 & 95.1 & 95.8 & 54.2 & 50.2 & 95.4 & 52.1 \\
	\midrule
	30 & PQAP & 71.3 & \textbf{91.1} & 64.7 & 52.3 & \textbf{85.9} & 77.2 & \textbf{0.79} & \textbf{0.16} & 0.63 & \textbf{95.9} & 70.0 & 58.3 & \textbf{98.3} & 80.9 & \textbf{73.2} \\
	& TNIa & \textbf{82.1} & 88.8 & 84.7 & \textbf{69.1} & 76.3 & \textbf{80.2} & 1.16 & 0.72 & \textbf{0.44} & 86.1 & 85.0 & \textbf{67.6} & 69.6 & \textbf{85.6} & 68.6\\
	& TNI & 78.4 & 89.3 & 78.0 & 61.8 & 79.3 & 72.9 & 1.20 & 0.60 & 0.59 & 85.3 & 73.3 & 54.8 & 71.9 & 78.9 & 62.2 \\
	& TNH & 78.9 & 83.8 & \textbf{86.0} & 66.9 & 63.0 & 77.5 & 1.49 & 1.01 & 0.48 & 81.4 & \textbf{87.3} & 66.4 & 55.6 & 84.2 & 60.5 \\
	& $d_H$ & 75.4 & 85.1 & 78.0 & 58.9 & 69.6 & 78.6 & 1.31 & 0.74 & 0.58 & 84.4 & 84.7 & 65.6 & 65.2 & 84.5 & 65.4 \\
	\midrule
	100/30 & PQAP & \textbf{92.8} & 100.0 & \textbf{92.8} & 0.0 & - & 94.8 & 1.24 & 0.64 & \textbf{0.60} & 97.8 & \textbf{96.7} & \textbf{61.5} & 70.6 & 97.2 & 65.8 \\
	& TNIa & 91.1 & 100.0 & 91.1 & 0.0 & - & \textbf{95.3} & \textbf{0.88} & \textbf{0.27} & 0.61 & \textbf{99.2} & 95.8 & 61.2 & \textbf{89.2} & \textbf{97.4} & \textbf{72.6} \\
	& TNI & 90.7 & 100& 90.7 & 0.0 & - & 94.7 & 1.09 & 0.38 & 0.72 & 98.7 & 95.5 & 58.3 & 83.1 & 97.1 & 68.5 \\
	& TNH & 90.1 & 100.0 & 90.1 & 0.0 & - & 92.0 & 2.13 & 1.05 & 1.08  & 96.6 & 94.8 & 44.1 & 54.9 & 95.7 & 48.9 \\ 
	& $d_H$ & 85.9 & 100.0 & 85.9 & 0.0 & - & 90.3 & 2.62 & 1.44 & 1.19 & 95.9 & 94.0 & 33.9 & 40.7 & 94.8 & 37.0 \\
	\bottomrule
	\end{tabular}
	\\[1ex]
	\caption{We report above how close \emph{classifications} of pairs of organoid images and \emph{clusterings} of organoid images are to decisions by biologists, on images not seen during training, of organoid classes seen during training (Test-100), organoid classes not seen during training (Test-30) and a combination (Test-100/30). 
	We compare the partial quadratic assignment problem (PQAP), the twin networks for images with and without data augmentation (TNI, TNIa), the twin networks for histograms (TNH) and the Hellinger distance ($d_H$).
	We abbreviate the accuracy (ACC), precision (P), recall (R), cuts (C), joins (J), Rand's index (RI), variation of information (VI) and F$_1$-score (F$_1$).}
	\label{table:independent-decisions}
\end{table}

Toward the classification of pairs of organoid images as belonging to the same or distinct clusters, we proceed in two steps:
Firstly, we learn the models defined in \Cref{section:models-assignment,section:models-twin-network,section:models-histogram} from the set \emph{Train-100} by the algorithms described in \Cref{section:algorithms-learning}.
Secondly, we apply these learned models in order to infer independently for pairs of images whether the images belong to the same or distinct clusters, for the sets \emph{Test-100}, \emph{Test-30} and the combined set \emph{Test-100/30}, by the algorithms described in \Cref{section:algorithms-inference}.
We report in \Cref{table:independent-decisions} and \Cref{figure:pr-curves} (left) how close these classifications are to classifications by biologists, on images not seen during learning, of organoid classes seen during learning (Test-100), organoid classes not seen during learning (Test-30) and a combination of both (Test-100/30).

It can be seen from \Cref{table:independent-decisions} that the accuracy is above 90\% for all models, for organoid classes seen during learning (\emph{Test-100}).
The accuracy is lower, about 70\%-80\% depending on the model, for organoid classes not seen during learning (\emph{Test-30}).
The accuracy is between these numbers for pairs of organoid images where one image is from a class seen during learning, and the other is not.
Closest to the truth by an accuracy of 98.9\% and 82.1\% on Test-100 and Test-30, respectively, are the classifications by the twin network learned with data augmentation.
The partial quadratic assignment problem is the most accurate model (92.8\%) for separating the organoid images in Test-100 from those in Test-30. 
It is also the least accurate model (71.3\%) on Test-30.
From the left column of \Cref{figure:pr-curves}, we see:
With respect to cuts, all models are accurate.
With respect to joins, both the twin network learned on cumulative histograms as well as the Hellinger distance are almost uninformative, while the other models are informative. 
With respect to Test-100, the twin network is the most accurate model across all metrics. 
With respect to Test-30, the partial quadratic assignment problem has a higher recall of joins and precision of cuts, by a margin of 9.6\% and 2.3\%, respectively. 
At the same time, the twin network has a higher precision of joins and recall of cuts, by a margin of 16.8\% and 20\%, respectively.
Computation times are discussed in \Cref{section:runtimes}.

\newcommand{\myaddplotpr}[6]{
	\addplot+[
		only marks,
		mark=#6,
		#5
	] table[
		x expr=\thisrow{#2}, 
		y expr=\thisrow{#3}, 
		col sep=comma
	] {#1};
	\ifthenelse { \equal {#4} {} }
		{}
		{\addlegendentry{\scriptsize #4};}
}

\begin{figure}
	\centering
	\small
	\begin{minipage}{0.45\linewidth}
		a)
		
		\begin{tikzpicture}
			\pgfplotsset{
				width=\linewidth,
				height=\linewidth
			}
			\begin{axis}[
					xlabel={Recall classif.~Test-100},
					ylabel={Precision classif.~Test-100},
					xmin=0,
					xmax=1.03,
					ymin=0,
					ymax=1.03,
					legend pos=south west,
					legend cell align={left}
				]
				\myaddplotpr{data/pr_curves/seen/data-geometry-joins-roc.csv}{recall}{precision}{PQAP joins}{\mycolorPQAP}{x}
				\myaddplotpr{data/pr_curves/seen/data-geometry-cuts-roc.csv}{recall}{precision}{PQAP cuts}{\mycolorPQAP}{+}	
				\myaddplotpr{data/pr_curves/seen/data-siamese-joins-roc-augment.csv}{recall}{precision}{TNIa joins}{\mycolorTNIa}{x}
				\myaddplotpr{data/pr_curves/seen/data-siamese-cuts-roc-augment.csv}{recall}{precision}{TNIa cuts}{\mycolorTNIa}{+}				
				\myaddplotpr{data/pr_curves/seen/data-siamese-joins-roc-no-augment.csv}{recall}{precision}{TNI joins}{\mycolorTNI}{x}
				\myaddplotpr{data/pr_curves/seen/data-siamese-cuts-roc-no-augment.csv}{recall}{precision}{TNI cuts}{\mycolorTNI}{+}
				\myaddplotpr{data/pr_curves/seen/data-siamese-joins-roc-cumulative-hist.csv}{recall}{precision}{TNH joins}{\mycolorTNH}{x}
				\myaddplotpr{data/pr_curves/seen/data-siamese-cuts-roc-cumulative-hist.csv}{recall}{precision}{TNH cuts}{\mycolorTNH}{+}
				\myaddplotpr{data/pr_curves/seen/data-histogram-joins-cuts-roc.csv}{recall_joins}{precision_joins}{$d_H$ joins}{\mycolorDH}{x}
				\myaddplotpr{data/pr_curves/seen/data-histogram-joins-cuts-roc.csv}{recall_cuts}{precision_cuts}{$d_H$ cuts}{\mycolorDH}{+}
			\end{axis}
		\end{tikzpicture}
	\end{minipage}%
	\begin{minipage}{0.45\linewidth}
		b)
		
		\begin{tikzpicture}
			\pgfplotsset{
				width=\linewidth,
				height=\linewidth
			}
			\begin{axis}[
					xlabel={Recall clustering Test-100},
					ylabel={Precision clustering Test-100},
					xmin=0,
					xmax=1.03,
					ymin=0,
					ymax=1.03,
					legend pos=south west,
					legend cell align={left}
				]
				\myaddplotpr{data/pr_curves/seen_clustered/data-siamese-roc-no-augment-clustered.csv}{recallJoins}{precisionJoins}{}{\mycolorTNI}{x}
				\myaddplotpr{data/pr_curves/seen_clustered/data-siamese-roc-no-augment-clustered.csv}{recallCuts}{precisionCuts}{}{\mycolorTNI}{+}
				\myaddplotpr{data/pr_curves/seen_clustered/data-siamese-roc-augment-clustered.csv}{recallJoins}{precisionJoins}{}{\mycolorTNIa}{x}
				\myaddplotpr{data/pr_curves/seen_clustered/data-siamese-roc-augment-clustered.csv}{recallCuts}{precisionCuts}{}{\mycolorTNIa}{+}
				\myaddplotpr{data/pr_curves/seen_clustered/data-siamese-roc-cumulative-hist-clustered.csv}{recallJoins}{precisionJoins}{}{\mycolorTNH}{x}
				\myaddplotpr{data/pr_curves/seen_clustered/data-siamese-roc-cumulative-hist-clustered.csv}{recallCuts}{precisionCuts}{}{\mycolorTNH}{+}
				\myaddplotpr{data/pr_curves/seen_clustered/data-geometry-roc-clustered.csv}{recall_joins}{precision_joins}{}{\mycolorPQAP}{x}
				\myaddplotpr{data/pr_curves/seen_clustered/data-geometry-roc-clustered.csv}{recall_cuts}{precision_cuts}{}{\mycolorPQAP}{+}
				\myaddplotpr{data/pr_curves/seen_clustered/data-histogram-joins-cuts-roc-clustered.csv}{recall_joins}{precision_joins}{}{\mycolorDH}{x}
				\myaddplotpr{data/pr_curves/seen_clustered/data-histogram-joins-cuts-roc-clustered.csv}{recall_cuts}{precision_cuts}{}{\mycolorDH}{+}
			\end{axis}
		\end{tikzpicture}
	\end{minipage}
	\begin{minipage}{0.45\linewidth}
		c)
		
		\begin{tikzpicture}
			\pgfplotsset{
				width=\linewidth,
				height=\linewidth
			}
			\begin{axis}[
					xlabel={Recall classif.~Test-30},
					ylabel={Precision classif.~Test-30},
					xmin=0,
					xmax=1.03,
					ymin=0,
					ymax=1.03,
					legend pos=south west,
					legend cell align={left}
				]		
				\myaddplotpr{data/pr_curves/unseen/data-siamese-joins-roc-no-augment-unseen.csv}{recall}{precision}{}{\mycolorTNI}{x}
				\myaddplotpr{data/pr_curves/unseen/data-siamese-cuts-roc-no-augment-unseen.csv}{recall}{precision}{}{\mycolorTNI}{+}
				\myaddplotpr{data/pr_curves/unseen/data-siamese-joins-roc-augment-unseen.csv}{recall}{precision}{}{\mycolorTNIa}{x}
				\myaddplotpr{data/pr_curves/unseen/data-siamese-cuts-roc-augment-unseen.csv}{recall}{precision}{}{\mycolorTNIa}{+}			
				\myaddplotpr{data/pr_curves/unseen/data-siamese-joins-roc-cumulative-hist-unseen.csv}{recall}{precision}{}{\mycolorTNH}{x}
				\myaddplotpr{data/pr_curves/unseen/data-siamese-cuts-roc-cumulative-hist-unseen.csv}{recall}{precision}{}{\mycolorTNH}{+}
				\myaddplotpr{data/pr_curves/unseen/data-geometry-joins-roc-unseen.csv}{recall}{precision}{}{\mycolorPQAP}{x}
				\myaddplotpr{data/pr_curves/unseen/data-geometry-cuts-roc-unseen.csv}{recall}{precision}{}{\mycolorPQAP}{+}
				\myaddplotpr{data/pr_curves/unseen/data-histogram-joins-cuts-roc-unseen.csv}{recall_joins}{precision_joins}{}{\mycolorDH}{x}
				\myaddplotpr{data/pr_curves/unseen/data-histogram-joins-cuts-roc-unseen.csv}{recall_cuts}{precision_cuts}{}{\mycolorDH}{+}				
			\end{axis}
		\end{tikzpicture}
	\end{minipage}%
	\begin{minipage}{0.45\linewidth}
		d)
		
		\begin{tikzpicture}
			\pgfplotsset{
				width=\linewidth,
				height=\linewidth
			}
			\begin{axis}[
					xlabel={Recall clustering Test-30},
					ylabel={Precision clustering Test-30},
					xmin=0,
					xmax=1.03,
					ymin=0,
					ymax=1.03,
					legend pos=south west,
					legend cell align={left}
				]
				\myaddplotpr{data/pr_curves/unseen_clustered/data-siamese-roc-no-augment-clustered-unseen.csv}{recallJoins}{precisionJoins}{}{\mycolorTNI}{x}
				\myaddplotpr{data/pr_curves/unseen_clustered/data-siamese-roc-no-augment-clustered-unseen.csv}{recallCuts}{precisionCuts}{}{\mycolorTNI}{+}
				\myaddplotpr{data/pr_curves/unseen_clustered/data-siamese-roc-augment-clustered-unseen.csv}{recallJoins}{precisionJoins}{}{\mycolorTNIa}{x}
				\myaddplotpr{data/pr_curves/unseen_clustered/data-siamese-roc-augment-clustered-unseen.csv}{recallCuts}{precisionCuts}{}{\mycolorTNIa}{+}
				\myaddplotpr{data/pr_curves/unseen_clustered/data-siamese-roc-cumulative-hist-clustered-unseen.csv}{recallJoins}{precisionJoins}{}{\mycolorTNH}{x}
				\myaddplotpr{data/pr_curves/unseen_clustered/data-siamese-roc-cumulative-hist-clustered-unseen.csv}{recallCuts}{precisionCuts}{}{\mycolorTNH}{+}
				\myaddplotpr{data/pr_curves/unseen_clustered/data-geometry-roc-clustered-unseen.csv}{recall_joins}{precision_joins}{}{\mycolorPQAP}{x}
				\myaddplotpr{data/pr_curves/unseen_clustered/data-geometry-roc-clustered-unseen.csv}{recall_cuts}{precision_cuts}{}{\mycolorPQAP}{+}
				\myaddplotpr{data/pr_curves/unseen_clustered/data-histogram-joins-cuts-roc-unseen-clustered.csv}{recall_joins}{precision_joins}{}{\mycolorDH}{x}
				\myaddplotpr{data/pr_curves/unseen_clustered/data-histogram-joins-cuts-roc-unseen-clustered.csv}{recall_cuts}{precision_cuts}{}{\mycolorDH}{+}
			\end{axis}
		\end{tikzpicture}
	\end{minipage}
	\begin{minipage}{0.45\linewidth}
		e)

		\begin{tikzpicture}
			\pgfplotsset{
				width=\linewidth,
				height=\linewidth
			}
			\begin{axis}[
					xlabel={Recall classif.~Test-100/30},
					ylabel={Precision classif.~Test-100/30},
					xmin=0,
					xmax=1.03,
					ymin=0,
					ymax=1.03,
					legend pos=south west,
					legend cell align={left}
				]
				\myaddplotpr{data/pr_curves/seen_unseen/data-siamese-joins-roc-no-augment-unseen+seen.csv}{recall}{precision}{}{\mycolorTNI}{x}
				\myaddplotpr{data/pr_curves/seen_unseen/data-siamese-cuts-roc-no-augment-unseen+seen.csv}{recall}{precision}{}{\mycolorTNI}{+}
				\myaddplotpr{data/pr_curves/seen_unseen/data-siamese-joins-roc-augment-unseen+seen.csv}{recall}{precision}{}{\mycolorTNIa}{x}
				\myaddplotpr{data/pr_curves/seen_unseen/data-siamese-cuts-roc-augment-unseen+seen.csv}{recall}{precision}{}{\mycolorTNIa}{+}			
				\myaddplotpr{data/pr_curves/seen_unseen/data-siamese-joins-roc-cumulative-hist-unseen+seen.csv}{recall}{precision}{}{\mycolorTNH}{x}
				\myaddplotpr{data/pr_curves/seen_unseen/data-siamese-cuts-roc-cumulative-hist-unseen+seen.csv}{recall}{precision}{}{\mycolorTNH}{+}
				\myaddplotpr{data/pr_curves/seen_unseen/data-geometry-joins-roc-combined.csv}{recall}{precision}{}{\mycolorPQAP}{x}
				\myaddplotpr{data/pr_curves/seen_unseen/data-geometry-cuts-roc-combined.csv}{recall}{precision}{}{\mycolorPQAP}{+}
				\myaddplotpr{data/pr_curves/seen_unseen/data-histogram-roc-combined.csv}{recall_joins}{precision_joins}{}{\mycolorDH}{x}
				\myaddplotpr{data/pr_curves/seen_unseen/data-histogram-roc-combined.csv}{recall_cuts}{precision_cuts}{}{\mycolorDH}{+}				
			\end{axis}
		\end{tikzpicture}
	\end{minipage}%
	\begin{minipage}{0.45\linewidth}
		f)
		
		\begin{tikzpicture}
			\pgfplotsset{
				width=\linewidth,
				height=\linewidth
			}
			\begin{axis}[
					xlabel={Recall clustering Test-100/30},
					ylabel={Precision clustering Test-100/30},
					xmin=0,
					xmax=1.03,
					ymin=0,
					ymax=1.03,
					legend pos=south west,
					legend cell align={left}
				]
				\myaddplotpr{data/pr_curves/seen_unseen_clustered/data-siamese-roc-no-augment-clustered-unseen+seen.csv}{recallJoins}{precisionJoins}{}{\mycolorTNI}{x}
				\myaddplotpr{data/pr_curves/seen_unseen_clustered/data-siamese-roc-no-augment-clustered-unseen+seen.csv}{recallCuts}{precisionCuts}{}{\mycolorTNI}{+}
				\myaddplotpr{data/pr_curves/seen_unseen_clustered/data-siamese-roc-augment-clustered-unseen+seen.csv}{recallJoins}{precisionJoins}{}{\mycolorTNIa}{x}
				\myaddplotpr{data/pr_curves/seen_unseen_clustered/data-siamese-roc-augment-clustered-unseen+seen.csv}{recallCuts}{precisionCuts}{}{\mycolorTNIa}{+}
				\myaddplotpr{data/pr_curves/seen_unseen_clustered/data-siamese-roc-cumulative-hist-clustered-unseen+seen.csv}{recallJoins}{precisionJoins}{}{\mycolorTNH}{x}
				\myaddplotpr{data/pr_curves/seen_unseen_clustered/data-siamese-roc-cumulative-hist-clustered-unseen+seen.csv}{recallCuts}{precisionCuts}{}{\mycolorTNH}{+}
				\myaddplotpr{data/pr_curves/seen_unseen_clustered/data-geometry-roc-clustered-combined.csv}{recall_joins}{precision_joins}{}{\mycolorPQAP}{x}
				\myaddplotpr{data/pr_curves/seen_unseen_clustered/data-geometry-roc-clustered-combined.csv}{recall_cuts}{precision_cuts}{}{\mycolorPQAP}{+}
				\myaddplotpr{data/pr_curves/seen_unseen_clustered/data-histogram-roc-combined-clustered.csv}{recall_joins}{precision_joins}{}{\mycolorDH}{x}
				\myaddplotpr{data/pr_curves/seen_unseen_clustered/data-histogram-roc-combined-clustered.csv}{recall_cuts}{precision_cuts}{}{\mycolorDH}{+}
			\end{axis}
		\end{tikzpicture}
	\end{minipage}
	\caption{Depicted above are precision recall curves for the independent classification of pairs of organoid images (left) and the clustering of organoid images (right), on the data sets Test-100 (top), Test-30 (middle) and Test-100/30 (bottom).}
	\label{figure:pr-curves}
\end{figure}
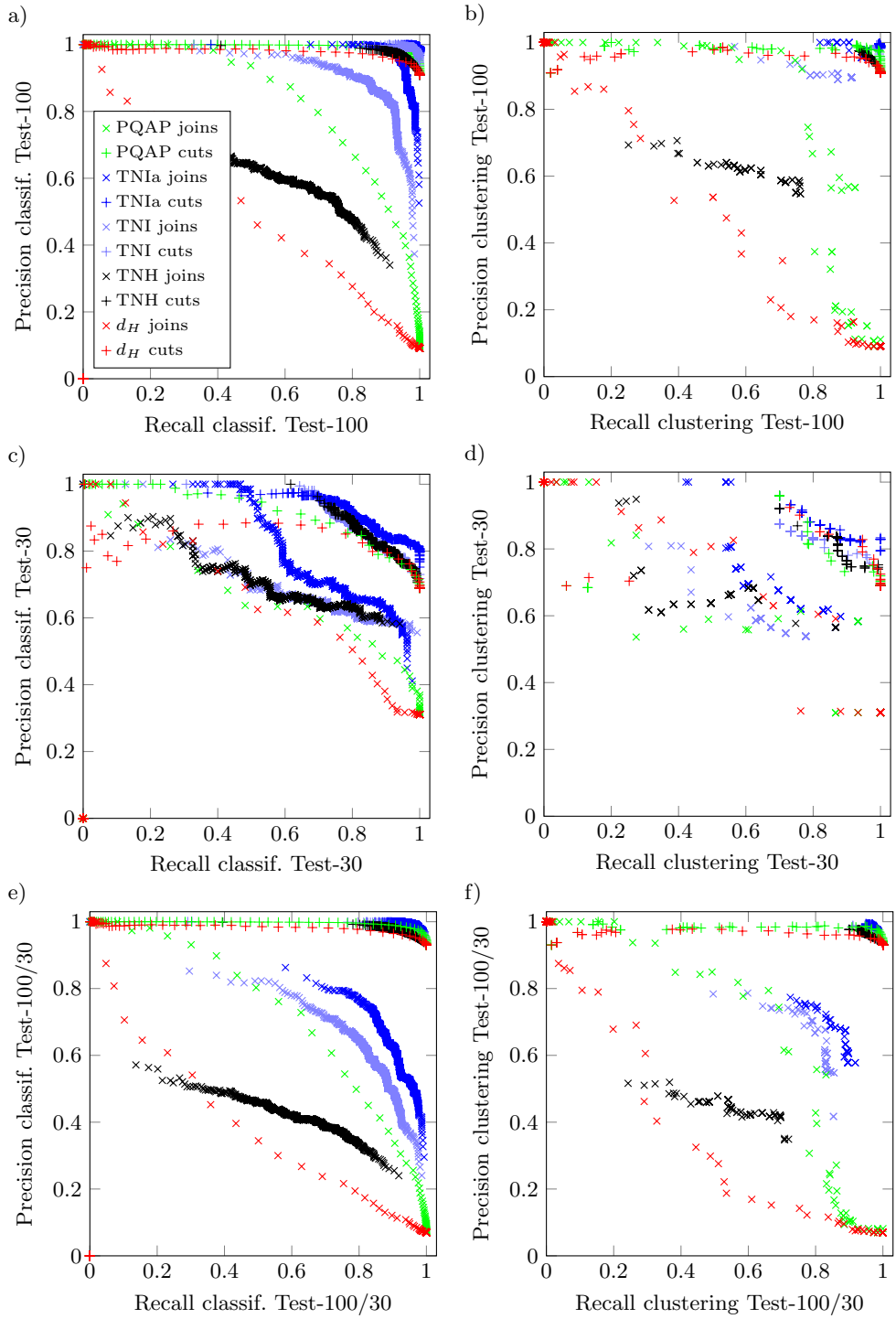

\subsection{Clustering organoid images}
\label{section:experiments-clustering}

\tikzmath{\gaaoffset=0.64;}
\tikzmath{\histooffset=0.81;}

\newcommand{\myaddplotvi}[6]{
	\addplot+[
		only marks,
		mark=#6,
		mark options={fill=#5, scale=0.4, solid},		
		#5
	] table[
		x expr=#2,
		y expr=#3, 
		col sep=comma
	] {#1};
	\ifthenelse { \equal {#4} {} }
		{}
		{\addlegendentry{\scriptsize #4};}
}

\begin{figure}
	\centering
	\small
	\begin{minipage}{0.31\linewidth}
		\centering
		\begin{tikzpicture}
			\pgfplotsset{
				width=1.2\linewidth,
				height=1.2\linewidth
			}
			\begin{axis}[
					title={Test-100},
					xlabel={$\chi$},
					ylabel={VI},
					xtick distance=1,
					xmin=-1.0,
					xmax=1.0,
					ymin=0,
					ymax=4.0,
					legend pos=north east,
					legend cell align={left}
				]
				\myaddplotvi{data/pr_curves/seen_clustered/data-geometry-roc-clustered.csv}
					{-(\thisrow{threshold}* 2.0 - 1.0 ) + (\gaaoffset * 2.0 - 1.0)}
					{\thisrow{VI}}
					{}{\mycolorPQAP}{*}
				\myaddplotvi{data/vi_curves/seen/data-siamese-roc-augment-clustering-test-2.csv}
					{\thisrow{logThreshold} / 13.8155}
					{\thisrow{VI}}
					{}{\mycolorTNIa}{*}
				\myaddplotvi{data/vi_curves/seen/data-siamese-roc-no-augment-clustering-test-2.csv}
					{\thisrow{logThreshold} / 13.8155}
					{\thisrow{VI}}
					{}{\mycolorTNI}{*}
				\myaddplotvi{data/vi_curves/seen/data-siamese-histograms-roc-clustering-test-2.csv}
					{\thisrow{logThreshold} / 13.8155}
					{\thisrow{VI}}
					{}{\mycolorTNH}{*}
				\myaddplotvi{data/pr_curves/seen_clustered/data-histogram-joins-cuts-roc-clustered.csv}
					{-(\thisrow{threshold}* 2.0 - 1.0) + (\histooffset * 2.0 - 1.0)}
					{\thisrow{VI}}
					{}{\mycolorDH}{*}
			\end{axis}
		\end{tikzpicture}
	\end{minipage}
	\hfill
	\begin{minipage}{0.31\linewidth}
		\centering
		\hspace{5ex}
		\begin{tikzpicture}
			\pgfplotsset{
				width=1.2\linewidth,
				height=1.2\linewidth
			}
			\begin{axis}[
					title={Test-30},
					xlabel={$\chi$},
					ylabel={\phantom{VI}},
					yticklabels=\empty,
					xtick distance=1,
					xmin=-1.0,
					xmax=1.0,
					ymin=0,
					ymax=4.0,
					legend pos=north east,
					legend cell align={left}
				]
				\myaddplotvi{data/pr_curves/unseen_clustered/data-geometry-roc-clustered-unseen.csv}
					{-(\thisrow{threshold} * 2.0 - 1.0) + (\gaaoffset * 2.0 - 1.0)}
					{\thisrow{VI}}
					{PQAP}{\mycolorPQAP}{*}
				\myaddplotvi{data/vi_curves/unseen/data-siamese-roc-augment-clustering-unseen-2.csv}
					{\thisrow{logThreshold} / 13.8155}
					{\thisrow{VI}}
					{TNIa}{\mycolorTNIa}{*}		
				\myaddplotvi{data/vi_curves/unseen/data-siamese-roc-no-augment-clustering-unseen-2.csv}
					{\thisrow{logThreshold} / 13.8155}
					{\thisrow{VI}}
					{TNI}{\mycolorTNI}{*}
				\myaddplotvi{data/vi_curves/unseen/data-siamese-roc-histograms-clustering-unseen-2.csv}
					{\thisrow{logThreshold} / 13.8155}
					{\thisrow{VI}}
					{TNH}{\mycolorTNH}{*}
				\myaddplotvi{data/pr_curves/unseen_clustered/data-histogram-joins-cuts-roc-unseen-clustered.csv}
					{-(\thisrow{threshold} * 2.0 - 1.0) + (\histooffset * 2.0 - 1.0)}
					{\thisrow{VI}}
					{$d_{H}$}{\mycolorDH}{*}
			\end{axis}
		\end{tikzpicture}
	\end{minipage}
	\hfill
	\begin{minipage}{0.31\linewidth}
		\centering
		\begin{tikzpicture}
			\pgfplotsset{
				width=1.2\linewidth,
				height=1.2\linewidth
			}
			\begin{axis}[
					title={Test-100/30},
					xlabel={$\chi$},
					ylabel={\phantom{VI}},
					yticklabels=\empty,
					xtick distance=1,
					xmin=-1.0,
					xmax=1.0,
					ymin=0,
					ymax=4.0,
					legend pos=outer north east,
					legend style={
						nodes={
							scale=0.55
						}
					},
					legend cell align={left}
				]
				\myaddplotvi{data/pr_curves/seen_unseen_clustered/data-geometry-roc-clustered-combined.csv}
					{-(\thisrow{threshold} * 2.0 - 1.0) + (\gaaoffset * 2.0 - 1.0)}
					{\thisrow{VI}}
					{}{\mycolorPQAP}{*}
				\myaddplotvi{data/vi_curves/seen_unseen/data-siamese-roc-augment-clustering-unseen+test-2.csv}
					{\thisrow{logThreshold} / 13.8155}
					{\thisrow{VI}}
					{}{\mycolorTNIa}{}
				\myaddplotvi{data/vi_curves/seen_unseen/data-siamese-roc-no-augment-clustering-unseen+test-2.csv}
					{\thisrow{logThreshold} / 13.8155}
					{\thisrow{VI}}
					{}{\mycolorTNI}{*}
				\myaddplotvi{data/vi_curves/seen_unseen/data-siamese-roc-histograms-clustering-unseen+test-2.csv}
					{\thisrow{logThreshold} / 13.8155}
					{\thisrow{VI}}
					{}{\mycolorTNH}{*}
				\myaddplotvi{data/pr_curves/seen_unseen_clustered/data-histogram-roc-combined-clustered.csv}
					{-(\thisrow{threshold} * 2.0 - 1.0) + (\histooffset * 2.0 - 1.0)}
					{\thisrow{VI}}
					{}{\mycolorDH}{*}
			\end{axis}
		\end{tikzpicture}
	\end{minipage}
	\caption{Depicted above is the variation of information distance between computed and true clusterings of organoid images as a function of a constant $\chi$ added to all cost coefficients of the correlation clustering problem.
	For this comparison, the costs from both the PQAP and the twin networks are scaled globally (not per instance) to $[-1,1]$, which does not alter the solutions.}
	\label{figure:vi-curves}
\end{figure}
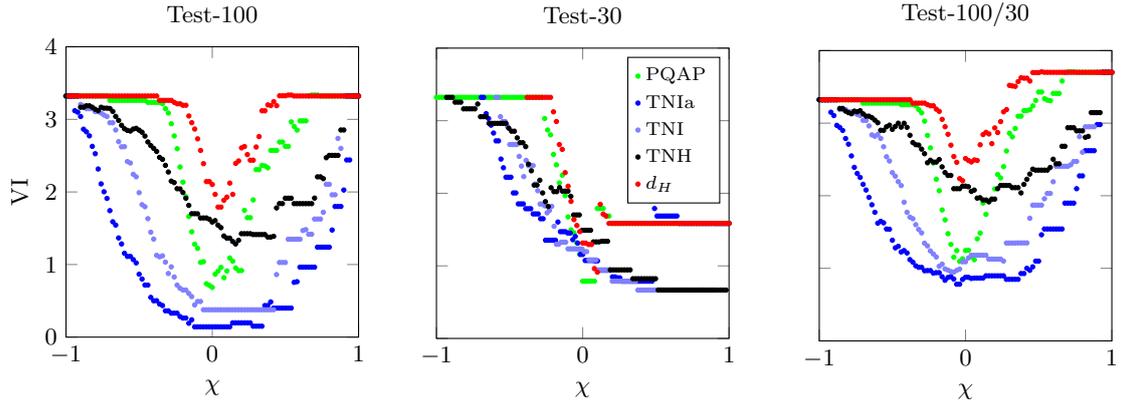
\begin{figure}
\centering
\includegraphics[width=\textwidth]{figures/figures_introduction/organoids_10_3_shuffled_bright.png}\\[0.15ex]
\includegraphics[width=\textwidth]{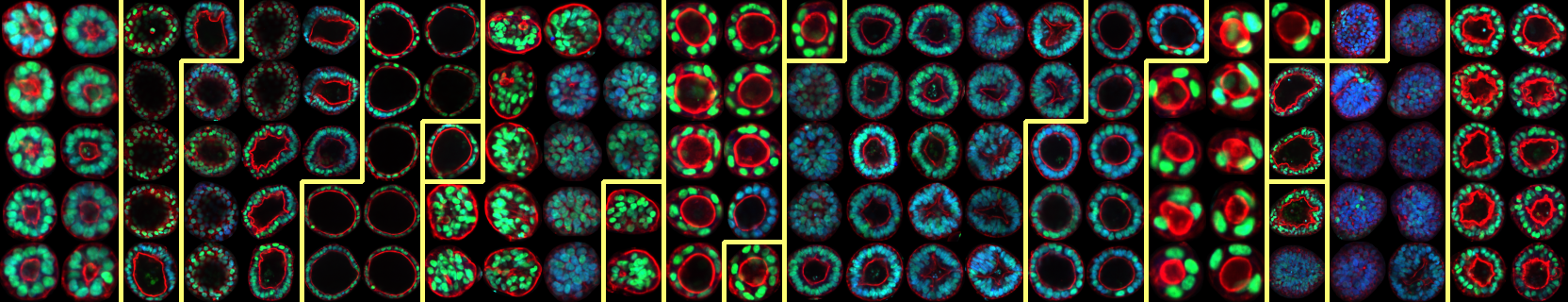}\\[0.15ex]
\includegraphics[width=\textwidth]{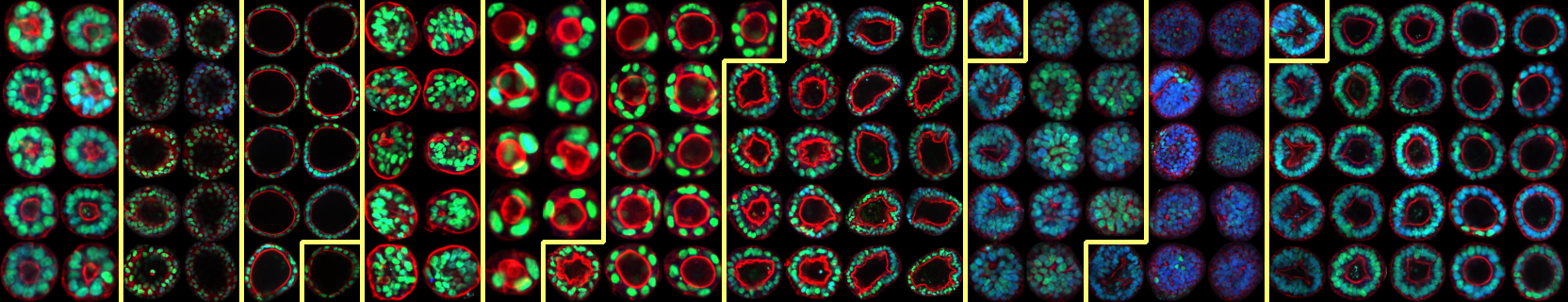}\\[0.15ex]
\includegraphics[width=\textwidth]{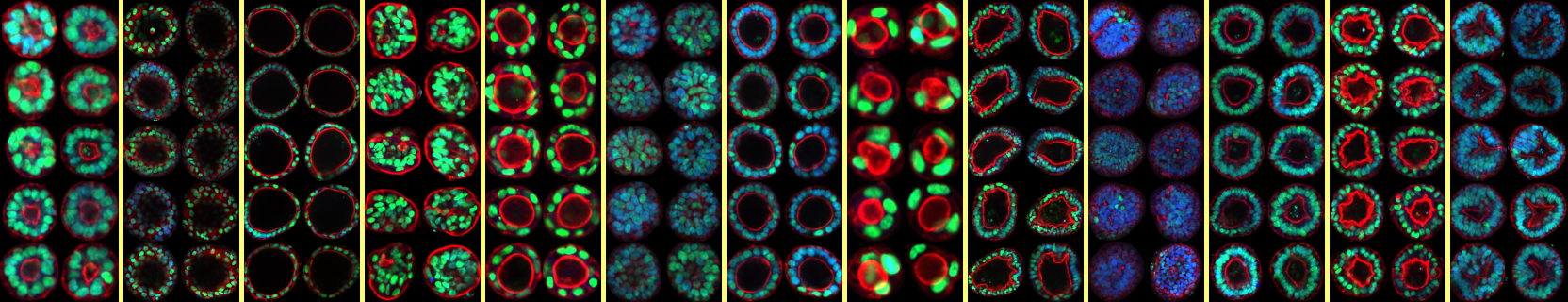}
\caption{Depicted above are,
in Row~1, all organoid images of the sets Test-100 and Test-30 (scaled differently to the same size for illustration),
in Row~2, the clustering of the images wrt.~costs defined by the partial quadratic assignment problem,
in Row~3, the clustering of the images wrt.~costs defined by a twin network, and
in Row~4, the clustering by biologists of Test-100 (Clusters 1--10) and Test-30 (Clusters 11--13).}
\label{figure:clustering_results}
\end{figure}

Toward the clustering of organoid images, we solve the instances of the correlation clustering problem defined in \Cref{section:models-clustering} by the algorithm described in \Cref{section:algorithms-inference}.
We report in \Cref{table:independent-decisions}, \Cref{figure:pr-curves} (right) and \Cref{figure:clustering_results} how close these clusterings are to clusterings by biologists, on images not seen during learning, of organoid classes seen during learning (Test-100), organoid classes not seen during learning (Test-30) and a combination of both (Test-100/30).
We report in \Cref{figure:vi-curves} how this distance between computed and true clusterings is affected by a constant added to all costs of the correlation clustering problem.
We discuss computation times in \Cref{section:runtimes}.

Closest to the truth for Test-100 by a variation of information \cite{arabie-1973,meila-2007} of $0.14$ is the twin network learned with data augmentation. 
For these images, it is consistently better than the other models with respect to all reported metrics.
Closest to the truth for Test-30 by a variation of information of $0.79$ is the partial quadratic assignment problem. 
It generalizes better from the classes of organoids seen during learning (Train-100 and Test-100) to the classes of images not seen during learning and contained in Test-30.
Comparing the accuracy of clustering and classification by comparing Rand's index (RI) \cite{rand-1971} with accuracy (ACC) in \Cref{table:independent-decisions}, we observe:
For Test-100, consolidating conflicting independent classifications of pairs of images by solving the correlation clustering problem increases the accuracy consistently, for all models. 
For Test-30, it increases the accuracy for the partial quadratic assignment problem and the Hellinger Distance while decreasing it for the twin networks.
It can be seen from \Cref{figure:vi-curves} that adding a constant to all costs of the correlation clustering problem determined by partial quadratic assignment does not improve solutions on any of the test sets.
In contrast, adding a positive constant to the costs determined by the twin network improves solutions on Test-30.
\section{Conclusion}
For the light microscopy images of organoids we have considered in this study, it can be seen from the experiments that both the partial quadratic assignment problem and the twin network can be applied, together with correlation clustering, to partition an image collection in a way that is similar to a partition of the same collection by biologists, by a Rand Index of $97.2\%$ and $99.4\%$, respectively, for classes of organoids seen during training, and by a Rand Index of $77.2\%$ and $80.2\%$, respectively, for classes of organoids not seen during training.
This enables unbiased quantification of heterogeneity and identification of morphological parameters of emerging patterns in biological systems.
This result cannot be considered a fully automatic solution to the problem but can be considered an informative proposal for computer-assisted clustering of organoid images. 
Quantifying the amount of work it can save compared to a fully manual clustering of organoid images is one direction of future research that connects with the field of human computer interaction.
Our study is limited to 2-dimensional light microscopy images of organoids. 
Generalizations to 3-dimensional images, also electron microscopy images, and possibly even sequences of such images can be considered in principle but are not discussed here.
Toward computer vision methods with potential for applications beyond organoid images, we have introduced a normalization that makes solutions to the partial quadratic assignment problem comparable across instances.
We contribute the organoid images with annotations as well as the instances from this application of the \textsc{np}-hard partial quadratic assignment problem and correlation clustering problem some of which are challenging even for the most sophisticated algorithms.
The source code of all algorithms and for reproducing the experiments is provided as supplementary material.
Ethical considerations are discussed in \Cref{section:ethics}.

\paragraph{Acknowledgements}
Organoids were cultured in Organoid and Stem Cell Facility, and experiments in 384-well plates and image acquisitions were performed with the support of the in-house screening facility (Technology Development Studio) of the Max Planck Institute of Molecular Cell Biology and Genetics.
B.A.~is supported by the Center of Scalable Data Analytics and AI, Dresden/Leipzig.
A.G.B.~and Y.H.K.~are supported by the Max Planck Society. 
R.K.~is supported by the Deutsche Forschungsgemeinschaft (DFG, German Research Foundation), Project 288034826-IRTG 2251 (to A.G.B.).

\appendix
\section{Appendix}

\subsection{Segmentation of organoids}
\label{appendix:organoid-segmentation}

\begin{figure}[t]
\centering
\small
\begin{tabular}{
c @{\,}
c @{\,}
c @{\,}
c @{\,}
c @{\,}
c @{\,}
c @{\,}
}
\includegraphics[width=0.16\textwidth]{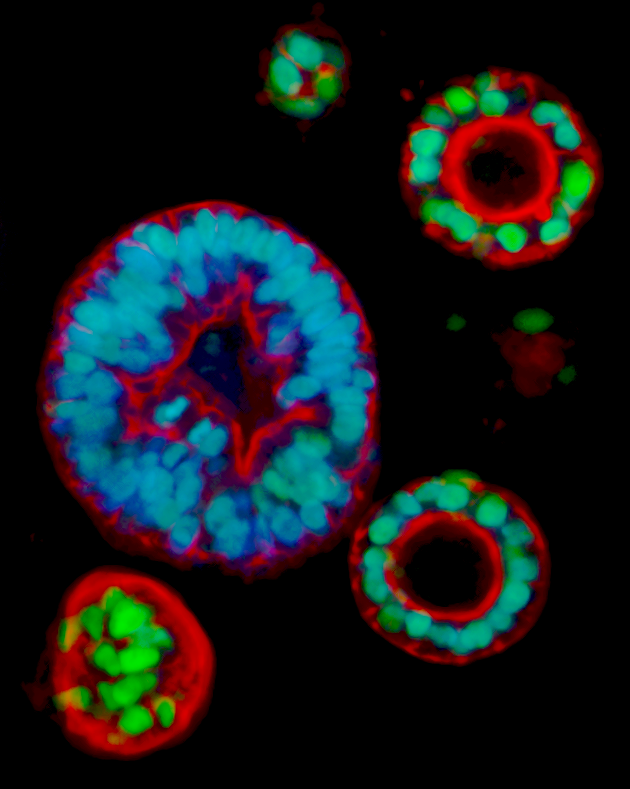} &
\includegraphics[width=0.16\textwidth]{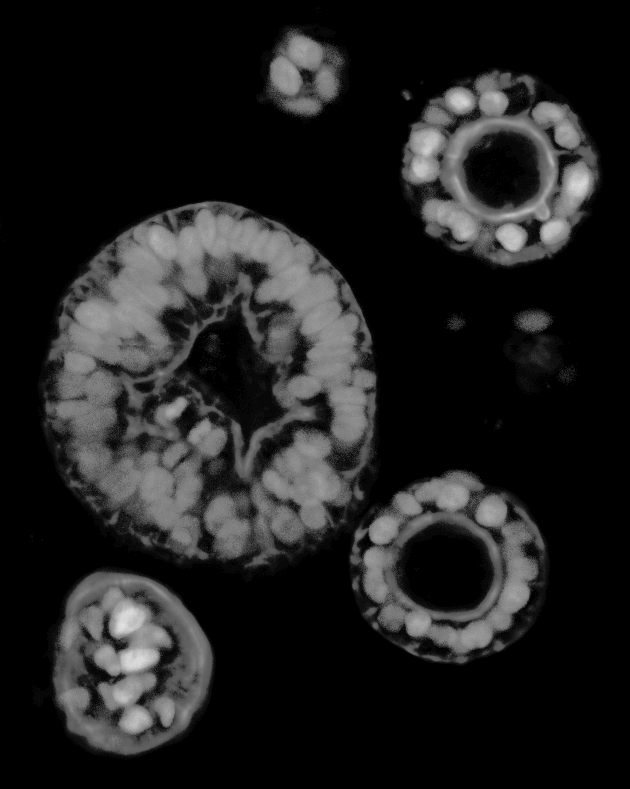} &
\includegraphics[width=0.16\textwidth]{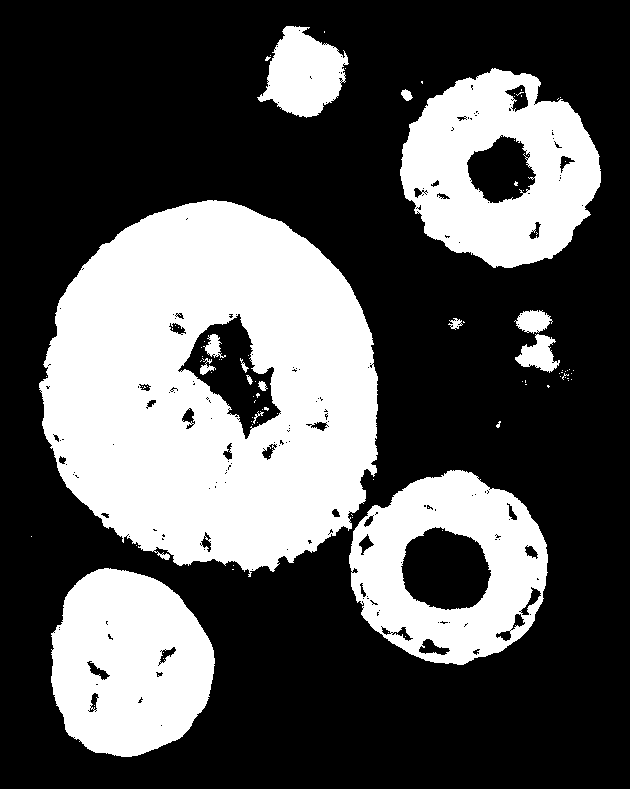} &
\includegraphics[width=0.16\textwidth]{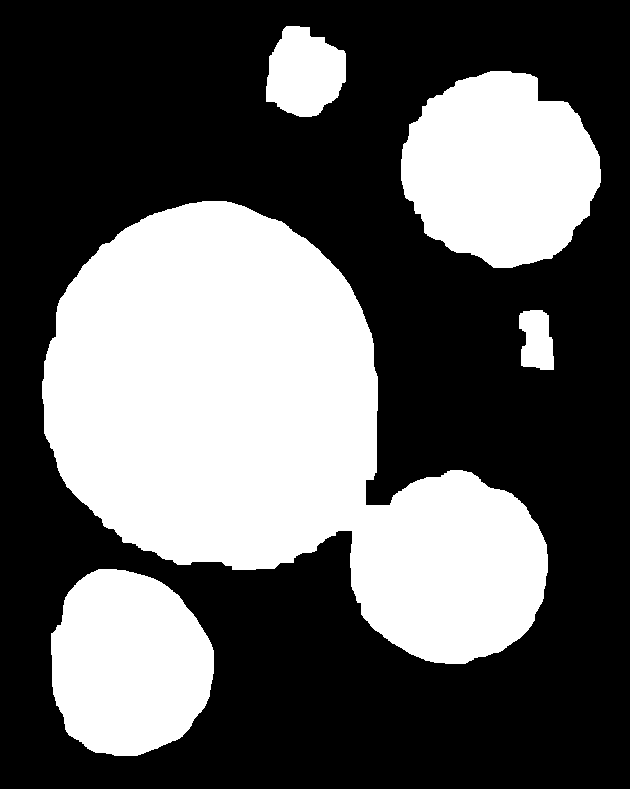} &
\includegraphics[width=0.16\textwidth]{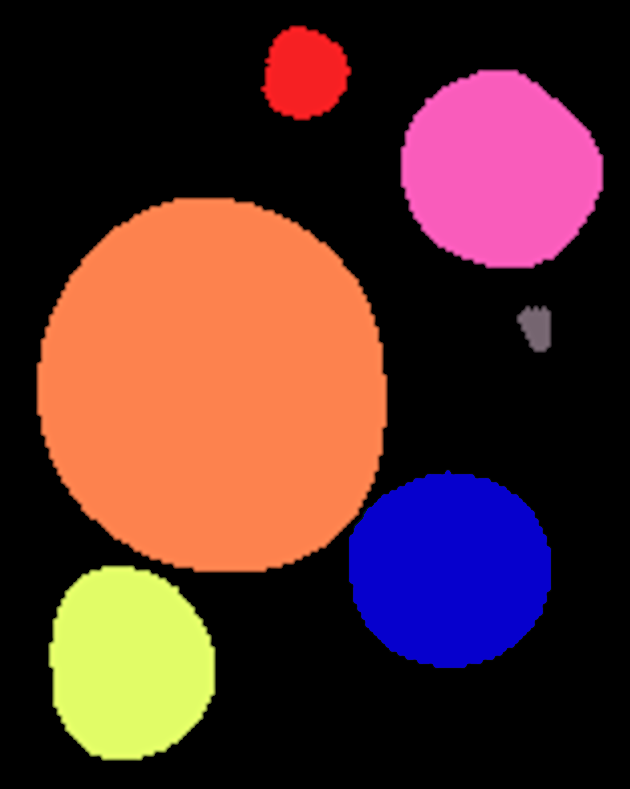} &
\includegraphics[width=0.16\textwidth]{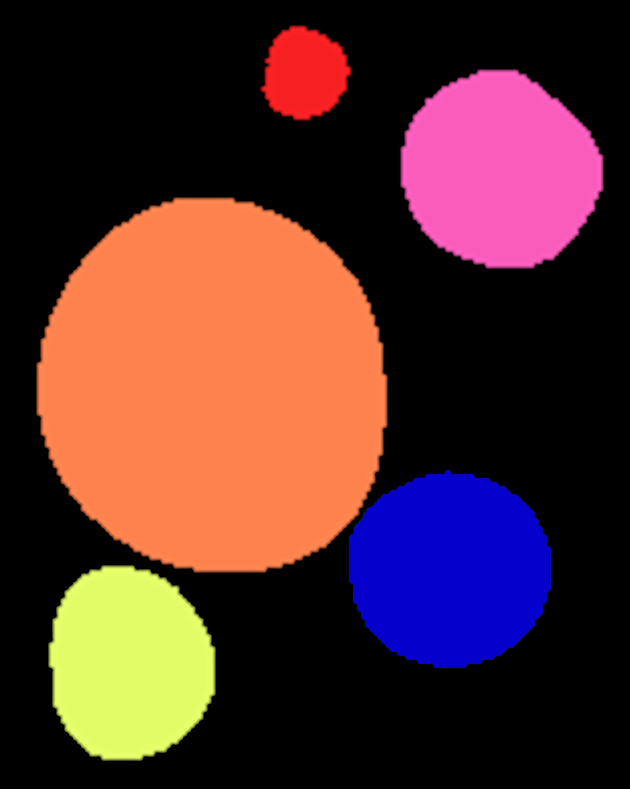}
\\
Raw &
Gray-scale &
Binary  &
Morphed & 
StarDist &
Threshold
\end{tabular}
\caption{Image pre-processing and organoid segmentation}
\label{appendix:fig_seg_pipeline}
\end{figure}

In order to extract images of individual organoids from microscopy images of numerous organoids within a field of view, we apply the StarDist method \cite{schmidt-2018}. 
Before applying StarDist, we process the images as depicted in \Cref{appendix:fig_seg_pipeline} and described below.
For further details, we refer to the source code.
Firstly, we convert the raw images with three color channels to gray-scale images. 
Secondly, we binarize the gray-scale images by applying a threshold. 
Thirdly, in order to remove noise while preserving the shapes and boundaries of the organoids, we apply a sequence of morphological operations.
The morphed images are then put into the StarDist method that outputs segments.
We suppress small segments.

\subsection{Extraction of key points}
\label{appendix:assignment-key-points}

For every organoid image $j$, we construct a set $V_{j}$ of key points as described below.
These key points are supposed to describe the organoid.

The green and the blue color channel show nuclei of cells within the organoid. 
We construct key points for these channels independently, as the brightness of nuclei in these channels is different.
To each channel, we apply the StarDist method \cite{schmidt-2018} that outputs segments. 
For each segment, we introduce one key point $v \in V_j$ at the barycenter $r_v^j$ of the segment and associate with this key point the average color $a_v^j$ across the segment.

The red color channel shows membranes within the organoid. 
For this channel, we construct key points in a greedy manner, prioritizing pixels based on the brightness, the brightness relative to all pixels, the brightness relative to all pixels in a local neighborhood, and the distance to key points introduced in previous iterations.
For further details, we refer to the source code.
We associate with each of these key points $v$ its location $r_v^j$ and color $a_v^j$.

\subsection{Estimation of organoid barycenter and extent}
\label{appendix:assignment-barycenter-extent}

We estimate the barycenter $r_0^j$ of the organoid in an image $j$ from the segment $S_j$ output by the StarDist method \cite{schmidt-2018} for this organoid, i.e.~from a set $S_j \in \mathbb{N}_0^j$ of pixels, as $r_0^j := |S|^{-1} \sum_{s \in S} s$.

We estimate the extent $\sigma_{0}^{j}$ of the organoid in an image $j$ from the segment $S_j$ and the set $V_j$ of key points (cf.~\Cref{appendix:assignment-key-points}) as
\begin{align}
\sigma_{0}^{j} = 
\max_{v \in V_{j}} \;
\| r_v^j - r_0^j \|_2 \;
\sup \;
\left\{ 
	\lambda \in \mathbb{R}	 
	\;\middle|\; 
	\lfloor
	r_0^j + \lambda \, ( r_v^j - r_0^j ) 	
	\rfloor
	\in S_j 
\right\}
\enspace .
\end{align}

\subsection{Motivation of assignment costs}
\label{section:geometric-consistency}

\begin{definition}
\label{definition:consistency}
We call a feasible assignment $x \in X_{V_j V_k}$ \emph{geometrically consistent} if there exist $r^* \in \mathbb{R}^2$, $s \in \mathbb{R}^+$ and $\gamma \in [0, 2\pi)$ for which the map $f \colon \mathbb{R}^2 \to \mathbb{R}^2$ such that for all $r \in \mathbb{R}^2$
\begin{align}
f(r) = r^* + s R_\gamma (r - r_0^j)
\quad \text{with} \quad
R_\gamma = \begin{pmatrix}\cos\gamma & -\sin\gamma\\ \sin\gamma & \cos\gamma\end{pmatrix}
\end{align}
has the following properties:
\begin{align}
f(r_0^j) & = r_0^k 
\label{eq:consistency-1}
\\
\forall r \in \mathbb{R}^2 \colon \quad
	| f(r) - r_0^k | & = | r - r_0^j | \, \sigma_0^k/\sigma_0^j
\label{eq:consistency-2}
\\
\forall (v,w) \in x^{-1}(1) \colon \hspace{9.6ex}
	f(r_v^j) & = r_w^k
\enspace .
\end{align}
\end{definition}

\begin{lemma}
For any geometrically consistent $x \in X_{V_j V_k}$ and any $r^*, s, \gamma$ according to \Cref{definition:consistency}, we have $r^* = r_0^k$ and $s = \sigma_0^k / \sigma_0^j$.
\end{lemma}

\begin{proof}
By \eqref{eq:consistency-1}, $r^* = r_0^k$.
Together with $s > 0$ follows
$|f(r) - r_0^k| = |r^* +  s R_\gamma (r - r_0^j) - r_0^k| = s |r - r_0^j|$.
Together with \eqref{eq:consistency-2} follows $s = \sigma_0^k / \sigma_0^j$.
\qed
\end{proof}

\begin{lemma}[Motivation of $d'_{vw}$ and $d''_{vwv'w'}$]
\label{lemma:consistency}
Let $x \in X_{V_j V_k}$ be geometrically consistent.
1.~For any pair $(v,w) \in V_j \times V_k$ with $x_{vw} = 1$, we have $d'_{vw} = 0$.
2.~For any distinct $v,v' \in V_j$ and any distinct $w,w' \in V_k$ such that $x_{vw} = x_{v'w'} = 1$, we have $d''_{vwv'w'} = 0$.
\end{lemma}

\begin{proof}
Firstly, $\sigma_w^k 
= |r_w^k - r_0^k| / \sigma_0^k 
= |f(r_v^j) - r_0^k| / \sigma_0^k
= |r_v^j - r_0^j| / \sigma_0^j 
= \sigma_w^j$.
Thus, $d'_{vw} = |\sigma_v^j = \sigma_w^k| = 0$.
Secondly, $\alpha^k_{v'w'} = 
\angle(r^k_{v'} - r^k_0, r^k_{w'} - r^k_0) 
= \angle(f(r^j_v) - r^k_0, f(r^j_w) - r^k_0)
= \angle(s R_\gamma (r^j_v - r^j_0), s R_\gamma (r^j_w - r^j_0))
= \angle(r^j_v - r^j_0, r^j_w - r^j_0)
= \alpha^j_{vw}$.
Thus, $d'_{vwv'w'} = |\alpha^j_{vw} - \alpha^k_{v'w'}| = 0$.
\qed
\end{proof}

\begin{lemma}[Bound on cost]
\label{lemma:bound}
For any instance of the partial quadratic assignment problem as defined above and any feasible solution $x$, we have
$- (1 - \lambda) (\theta \delta + (1-\theta) \delta') - \lambda \delta'' \leq \varphi_{c^{jk}}(x)$.
For any solution $x^*$, also $\varphi_{c^{jk}}(x^*) \leq 0$.
\end{lemma}

\begin{proof}
Obviously, $\varphi_{\bar c^{jk}}(\bar x) \leq \varphi_{c^{jk}}(x)$ if three conditions hold:
Firstly, $\bar x$ is such that $\bar x_{vw} = 1$ for $n^{jk}_1$ many point pairs $(v,w) \in V_j \times V_k$.
Secondly, $\bar c^{jk}$ is such that for all $(v,w) \in V_j \in V_k$ with $x_{vw} = 1$, $\bar c^{jk}_{vw}$ is maximally negative.
This happens if and only if $d_{vw} = d'_{vw} = 0$, i.e.~if $\bar c^{jk}_{vw} = - \theta \delta - (1-\theta) \delta'$.
Thirdly, $\bar c^{jk}$ is such that for all distinct $v,v' \in V_j$ and all distinct $w,w' \in V_k$, $\bar c^{jk}_{vwv'w'}$ is maximally negative.
This happens if and only if $d'_{vwv'w'} = 0$, i.e.~if $\bar c^{jk}_{vwv'w'} = - \delta''$.
In this case, $\varphi_{\bar c^{jk}}(\bar x) = - (1 - \lambda) (\theta \delta + (1-\theta) \delta') - \lambda \delta''$, by definition of $n^{jk}_1$ and $n^{jk}_2$.

The feasible solution $0 \in X_{V_j V_k}$ is such that $\varphi_{c^{jk}}(0) = 0$.
This implies for any solution $x^*$ that $\varphi_{c^{jk}}(x^*) \leq 0$.
\qed
\end{proof}

\begin{figure}
\centering
\includegraphics[width=0.9\textwidth]{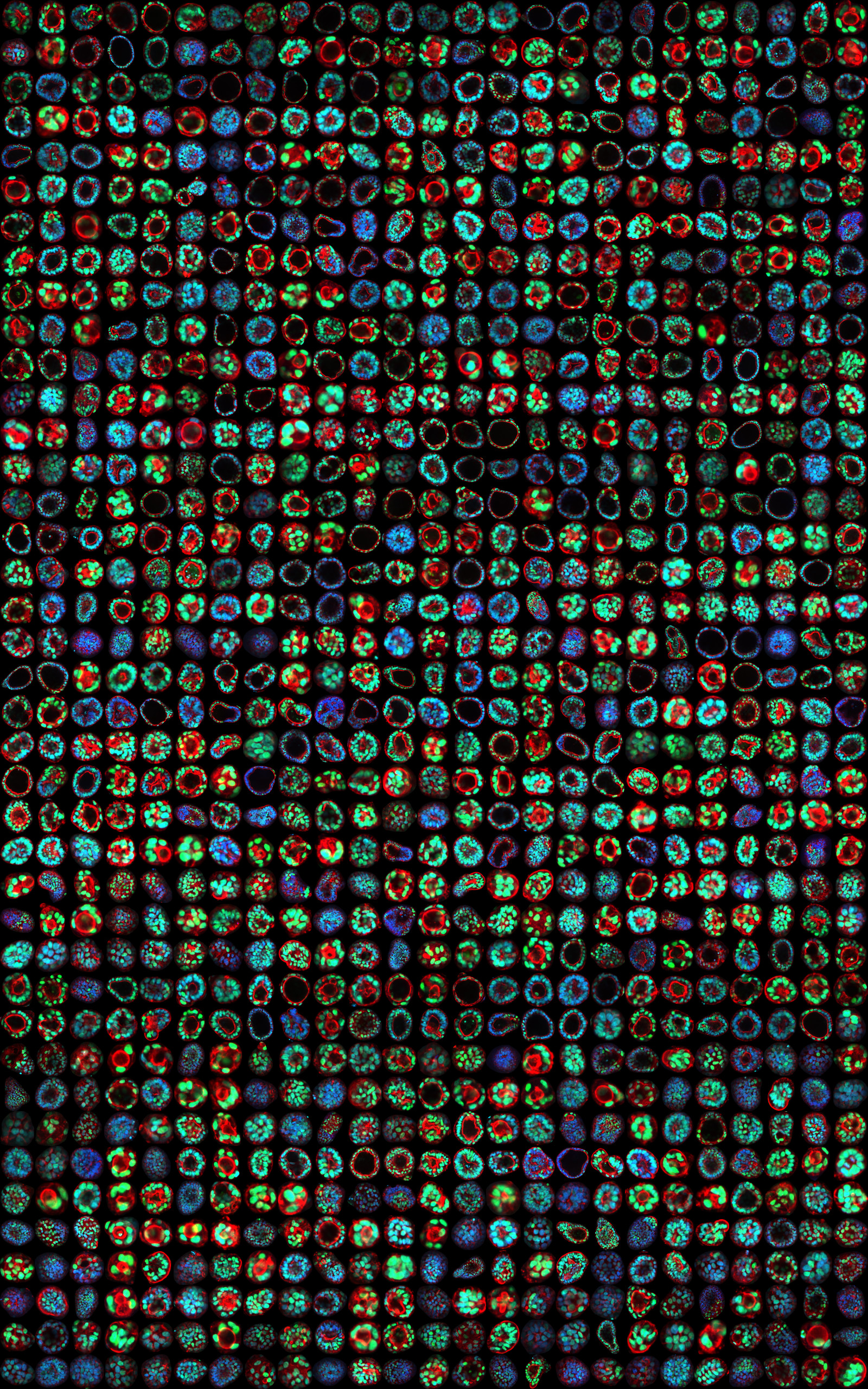}
\caption{Depicted above is the set \textit{Unlabeled-1000} of 1000 organoid images.}
\label{appendix:clustering_1000_unordered_figure}
\end{figure}

\begin{figure}
\centering
\includegraphics[width=0.9\textwidth]{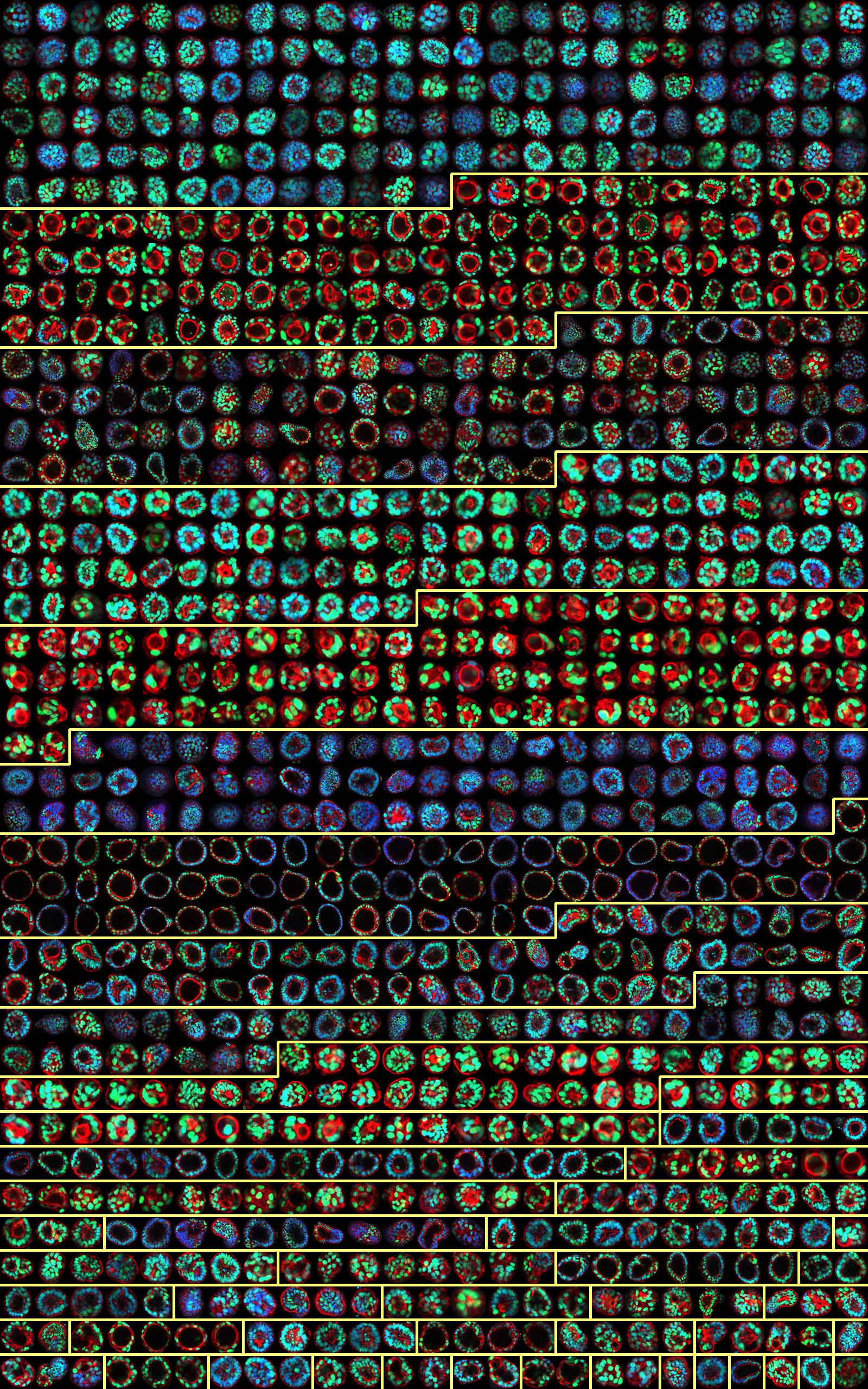}
\caption{Depicted above is the correlation clustering of the set \textit{Unlabeled-1000} with respect to costs defined by the twin network learned with data augmentation.}
\label{appendix:clustering_1000_NN_figure}
\end{figure}

\begin{figure}
\centering
\includegraphics[width=0.9\textwidth]{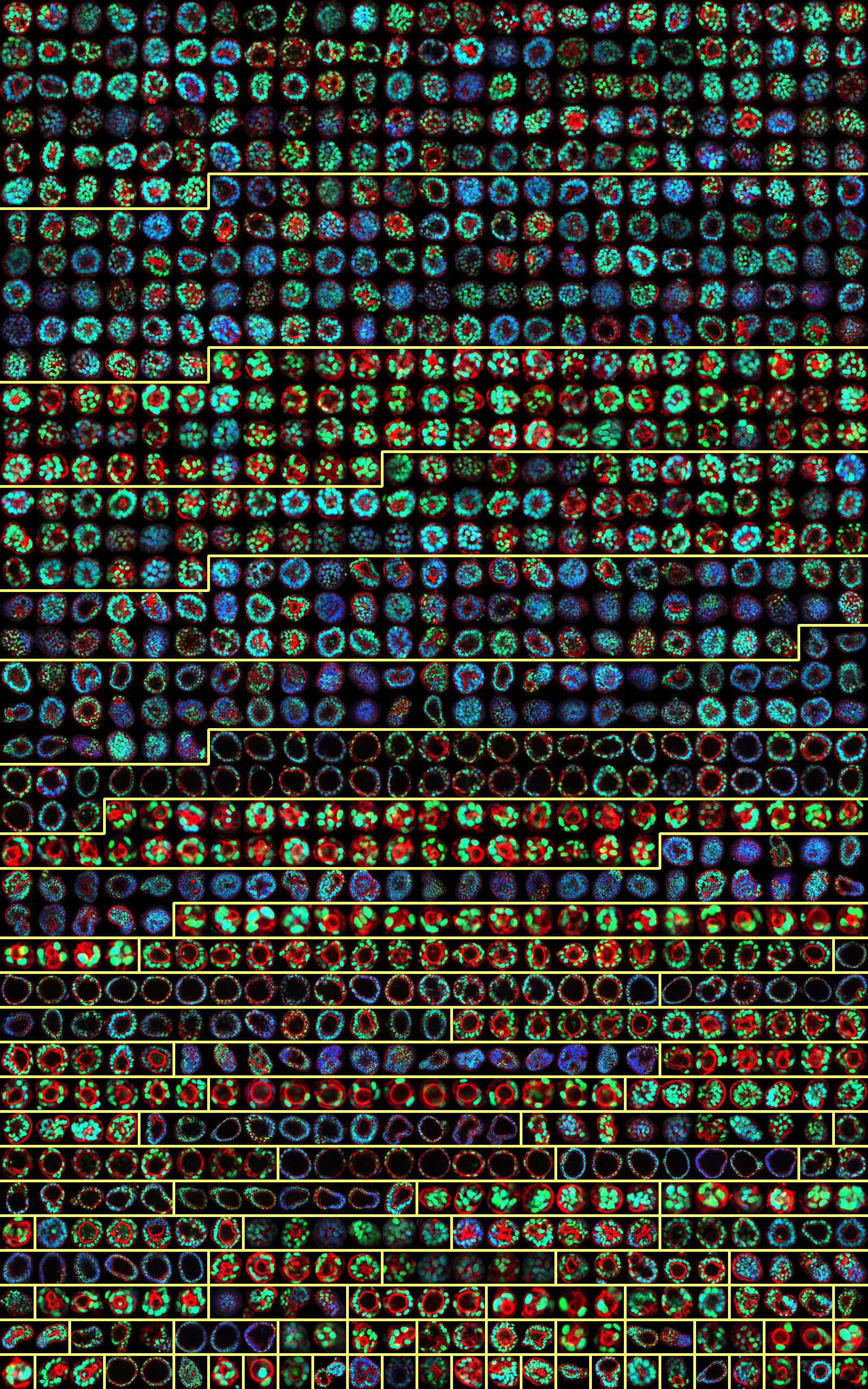}
\caption{Depicted above is the correlation clustering of the set \textit{Unlabeled-1000} with respect to costs defined by local solutions to the partial quadratic assignment problem.}
\label{appendix:clustering_1000_GaA_figure}
\end{figure}

\subsection{Additional qualitative experiments}
\label{appendix:experiments-qualitative}

In order to report, qualitatively, the clustering of a larger collection of organoid images by the methods under examination, we show 
in \Cref{appendix:clustering_1000_unordered_figure} the set \emph{Unlabeled-1000} of 1000 organoid images, 
in \Cref{appendix:clustering_1000_NN_figure} the correlation clustering of these images with respect to costs defined by the twin network learned with data augmentation, and 
in \Cref{appendix:clustering_1000_GaA_figure} the correlation clustering of \emph{Unlabeled-1000} with respect to costs defined by local solutions to the partial quadratic assignment problem.
It can be seen from these figures that correlation clustering with respect to costs defined by local solutions to the partial quadratic assignment problem leads to more and smaller clusters.

\subsection{Runtimes}
\label{section:runtimes}

\tikzmath{\tominutesfactor=60000000;}

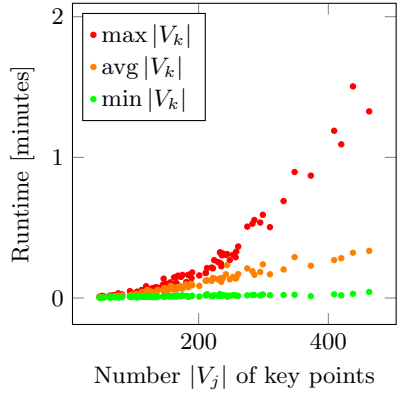
\begin{figure}[t]
\small
\noindent\makebox[\textwidth][c]{%
	\begin{minipage}{0.4\linewidth}
		\begin{tikzpicture}
			\pgfplotsset{
				width=\linewidth,
				height=\linewidth
			}
						\begin{axis}[
				xlabel={Number $|V_j|$ of key points},
				ylabel={Runtime [minutes]},
				legend pos=north west,
				legend cell align={left}
				]
				\addplot+[
				only marks,
				mark=*,
				mark options={fill=red, scale=0.5, solid},
				red
				] table[
				x expr=\thisrow{num_features}, 
				y expr=\thisrow{cost_runtime_max} / \tominutesfactor, 
				col sep=comma
				] {data/runtime_measurements/runtime_measurements_GMP.csv};
				\addlegendentry{$\max |V_k|$};				
				
				\addplot+[
				only marks,
				mark=*,
				mark options={fill=orange, scale=0.5, solid},
				orange
				] table[
				x expr=\thisrow{num_features}, 
				y expr=\thisrow{cost_runtime_avg} / \tominutesfactor, 
				col sep=comma
				] {data/runtime_measurements/runtime_measurements_GMP.csv};
				\addlegendentry{$\mathrm{avg} \, |V_k|$};
				
				\addplot+[
				only marks,
				mark=*,
				mark options={fill=green, scale=0.5, solid},
				green
				] table[
				x expr=\thisrow{num_features}, 
				y expr=\thisrow{cost_runtime_min} / \tominutesfactor, 
				col sep=comma
				] {data/runtime_measurements/runtime_measurements_GMP.csv};
				\addlegendentry{$\min |V_k|$};
			\end{axis}
		\end{tikzpicture}
	\end{minipage}}%
	\caption{Reported above are the runtimes of the local search algorithm for the partial quadratic assignment problem, for all instances from \emph{Test-100}. 
	For every number $|V_j|$ of key points in the first image, the runtime is shown for the maximum number $\max |V_k|$ of key points in the second image, the average $\mathrm{avg} \, |V_k|$ and the minimum $\min |V_k|$.}
	\label{figure:runtime-curves}
\end{figure}

\textbf{Partial quadratic assignment problem.}
We report in \Cref{figure:runtime-curves} the runtimes of the local search algorithm for the partial quadratic assignment problem, for all instances from \emph{Test-100}. 
For every number $|V_j|$ of key points in the first image, the runtime is shown for the maximum number $\max |V_k|$ of key points in the second image, the average $\mathrm{avg} \, |V_k|$ and the minimum $\min |V_k|$.
These runtimes are observed on a single core of an Intel Xeon Platinum 8470 CPU equipped with 512~GB of RAM, operating at 2.00~GHz.

In order to adjust the five parameters $\delta, \delta', \delta'', \theta$ and $\lambda$ of the instances of the partial quadratic assignment problem, we performed simulated annealing for 8~hours on one core of the same CPU. 

\textbf{Correlation clustering problem.}
All instances of the correlation clustering problem have been solved to optimality using Gurobi on all 52 cores of the same CPU.
For costs defined by the twin network with images as input, this has taken at most 20 minutes.
For costs defined by the twin network with histograms as input, it has taken at most 42 minutes. 
For costs defined by local solutions to the partial quadratic assignment problem, it has taken at most 14 minutes.
For costs based on the Hellinger distance, it has taken at most 73 minutes. 

\textbf{Twin networks.}
Learning the twin network with images as input, for 6,000 iterations, takes about 20~minutes; learning the twin network with histograms as input, for 6,000 iterations, takes about 70~minutes, both on a single NVIDIA GeForce RTX 4900.

\subsection{Ethical considerations}
\label{section:ethics}

This article presents work whose goal is to advance the fields of computer vision and biomedical image analysis.
There are many potential societal consequences of this work, none which we feel must be specifically highlighted here.

The H9 human embryonic stem cell (hESC) line was purchased from WiCell, and GFP reporter for PDX1 was generated in the Grapin-Botton lab. 
Approval to work on hESCs was obtained from the Robert Koch Institute under reference number 3.04.02/0148, and the hESC work was conducted according to the guidelines. 
Pancreatic organoids were generated using hESC-derived pancreatic progenitors, as previously reported \cite{goncalves-2021}.

\bibliographystyle{abbrvurl}
\bibliography{arxiv}
\end{document}